\documentclass{article}

\usepackage[preprint,nonatbib]{neurips_2023}

\usepackage[utf8]{inputenc} 
\usepackage[T1]{fontenc}    
\usepackage{hyperref}       
\usepackage{url}            
\usepackage{booktabs}       
\usepackage{amsfonts}       
\usepackage{nicefrac}       
\usepackage{microtype}      
\usepackage{xcolor}         
\usepackage{todonotes}
\usepackage{booktabs} 
\usepackage{bm}
\usepackage{xcolor}
\usepackage[most]{tcolorbox}
\usepackage{nicefrac}
\usepackage[export]{adjustbox}
\usepackage{enumitem}
\usepackage{wrapfig}
\usepackage{caption}[small]
\usepackage{amsmath}
\usepackage{amsthm}
\usepackage{graphicx}
\usepackage{subcaption}
\usepackage{graphicx}
\usepackage{caption}
\usepackage{wrapfig}

\usepackage[ruled,linesnumbered,noresetcount,vlined]{algorithm2e}

\newtheorem{problem*}{Problem}

\newtheorem{theorem}{Theorem}

\newtheorem{definition}{Definition}
\newtheorem{proposition}{Proposition}

 \newcommand{\cY}{\mathcal{Y}}
 \newcommand{\cX}{\mathcal{X}}

\DeclareMathOperator*{\argmax}{argmax}

\theoremstyle{plain}

\definecolor{myblue}{RGB}{0, 51, 153}

\newcommand{\hl}[1]{\textcolor{myblue}{\emph{#1}}}
\newcommand{\bhl}[1]{\textbf{\hl{#1}}}

\newcommand{\blue}[1]{\textcolor{myblue}{#1}}
\newcommand{\red}[1]{\textcolor{purple}{#1}}

\title{Data Minimization at Inference Time}

\author{%
  Cuong Tran\\
  Syracuse University\\
  \texttt{cutran@syr.edu}
  \And
  Ferdinando Fioretto\\
  Syracuse University\\
  \texttt{ffiorett@syr.edu}
}

\allowdisplaybreaks
\begin{document}
\maketitle

\begin{abstract}
In domains with high stakes such as law, recruitment, and healthcare, learning models frequently rely on sensitive user data for inference, necessitating the complete set of features. This not only poses significant privacy risks for individuals but also demands substantial human effort from organizations to verify information accuracy. This paper asks whether it is necessary to use \emph{all} input features for accurate predictions at inference time. The paper demonstrates that, in a personalized setting, individuals may only need to disclose a small subset of their features without compromising decision-making accuracy. The paper also provides an efficient sequential algorithm to determine the appropriate attributes for each individual to provide. Evaluations across various learning tasks show that individuals can potentially report as little as 10\% of their information while maintaining the same accuracy level as a model that employs the full set of user information.
\end{abstract}

\section{Introduction}
\label{sec:intro}

The remarkable success of learning models also brought with it pressing challenges at the interface of privacy and decision-making. Privacy, in particular, has been cited as one of the most pressing challenges of modern machine learning systems \cite{papernot2016towards}. 
The requirement to protect personally identifiable information is especially important as machine learning systems become increasingly adopted to guide consequential decisions in legal processes, banking, hiring, and healthcare. 

To contrast these challenges, several privacy-enhancing technologies have been proposed in the last decades. 
However, current research on privacy mechanisms, including differential privacy (DP) \cite{Dwork:06}, mainly aims to protect information within training data, potentially leaving user information vulnerable during system deployment. Conventionally, users must disclose their complete set of features for inference, even if not all may be necessary for accurate predictions. This practice not only presents significant privacy risks for users but also burdens companies and organizations with an extensive human effort to verify the accuracy of disclosed information (e.g., auditing in finance operations). Additionally, such an approach may violate the EU General Data Protection Regulation's principle of \emph{data minimization}, which states that personal data should be "adequate, relevant, and limited to what is necessary in relation to the purposes for which they are processed" \cite{rastegarpanah2021auditing,regulation2016regulation}.

This paper challenges this setting and asks whether it is necessary to require \emph{all} input features for a model to produce accurate or nearly accurate predictions during inference. We refer to this question as the \emph{data minimization for inference} problem. 
This unique question bears profound implications for privacy in model personalization,  which often necessitates the disclosure of substantial user data.  
We show that, under a personalized setting, each individual may only need to release a small subset of their features to achieve the \emph{same} prediction errors as those obtained when all features are disclosed. 
Following this result, we also provide an efficient sequential algorithm to identify the minimal set of attributes that each individual should reveal. Evaluations across various learning tasks indicate that individuals may be able to report as little as 10\% of their information while maintaining the same accuracy level as a model using the full set of user information.

\textbf{Contributions.} In summary, the paper makes the following contributions: 
{\bf (1)} it initiates a study to analyze the optimal subset of data features that each individual should disclose at inference time in order to achieve the same accuracy as if all features were disclosed; 
{\bf (2)} it links this analysis to a new concept of \emph{data minimization for inference} in relation to privacy, 
{\bf (3)} it proposes theoretically motivated and efficient algorithms for determining the minimal set of attributes each individual should provide to minimize their data;
and 
{\bf (4)} it conducts a comprehensive evaluation illustrating that individuals may be able to report as little as 10\% of their information to ensure the same level of accuracy of a model that uses the complete set of user information.

To the best of our knowledge, this is the first work studying the connection between data minimization and accuracy at inference time. 
While we are not aware of studies on data minimization for inference problems, we discuss in Appendix \ref{app:related_work} connections with DP, feature selection, and active learning.

\section{Settings and objectives}
\label{sec:settings}

We consider a dataset $D$ consisting of samples $(x, y)$ drawn from an unknown distribution $\Pi$. Here, $x \in \cX$ is a feature vector and $y \in \mathcal{Y} = [L]$ is a label with $L$ classes. The features in $x$ are categorized into \emph{public} $x_P$ and \emph{sensitive} $x_S$ features, with their respective indexes in vector $x$ denoted as $P$ and $S$, respectively. 
We consider classifiers $f_\theta : \mathcal{X} \to \mathcal{Y}$, which are trained on a public dataset from the same data distribution $\Pi$ above. The classifier produces a score 
$\tilde{f}_\theta(x) \in \mathbb{R}^L$ over the classes and a final output class, $f_\theta(x) \in [L]$, given input $x$.
The model's outputs $f_\theta(x)$ and $\tilde{f}_\theta(x)$ are also often referred to as hard and soft predictions, respectively. 

Without loss of generality, we assume that all features in $\mathcal{X}$ lie within the range $[-1, 1]$. In this setting, we are given a trained model $f_\theta$ and, at inference time, we have access to the public features $x_P$. These features might be revealed through user queries or collected by the provider during previous interactions. Our focus is on the setting where $|S| \ll |D|$, and for simplicity, the following considers binary classifiers, where $\cY = 
\{0, 1\}$ and $\tilde{f}_{\theta} \in \mathbb{R}$. Multi-class settings are addressed in Appendix \ref{app:multiclass}.

In this paper, the term \textbf{data leakage} of a model, refers to the percentage of sensitive features that are revealed unnecessarily, meaning that their exclusion would not significantly impact the model's output. 
\hl{Our goal is to design algorithms that accurately predict the output of the model using the smallest possible number of sensitive features}, \emph{thus minimizing the data leakage at inference time}. This objective reflects our desire for privacy.

\begin{wrapfigure}[8]{r}{150pt}
\vspace{-28pt}
\centering
    \includegraphics[width=\linewidth]{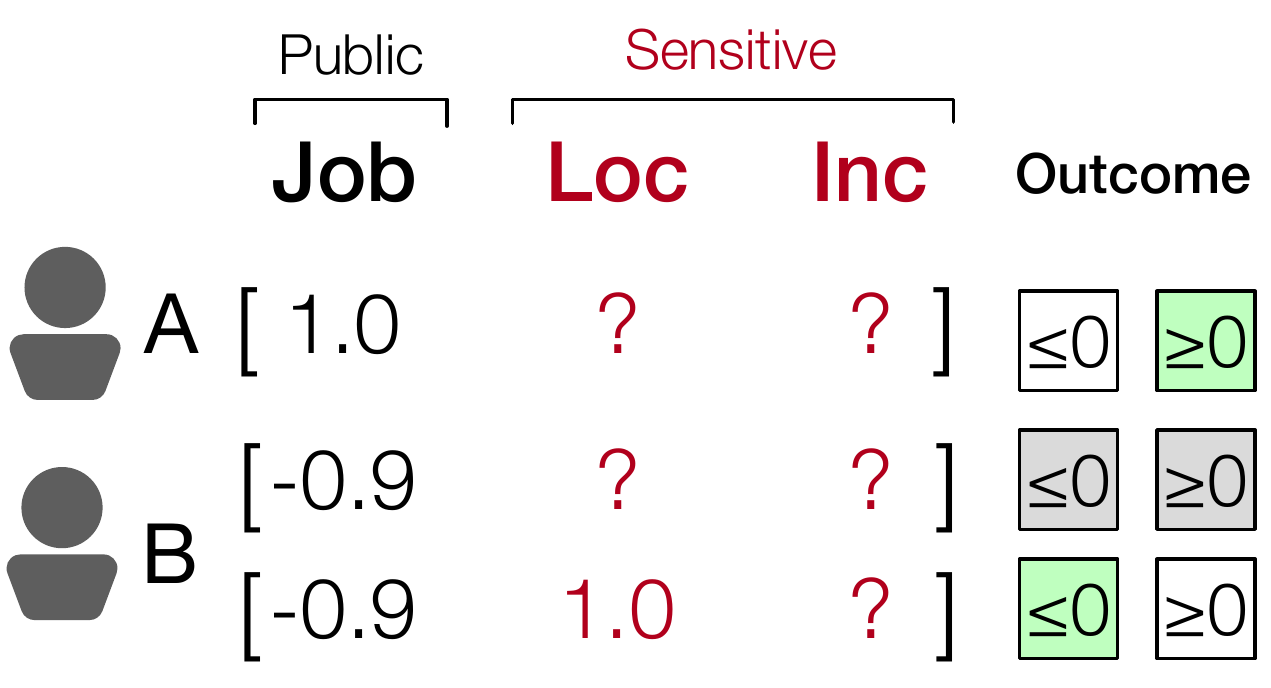} 
    \vspace{-12pt}
    \caption{\small Motivating example.}
    \label{fig:motivation_A}
\end{wrapfigure}

To clarify key points discussed in the paper, let us consider a loan approval task (Figure \ref{fig:motivation_A}) where individual features are represented by the set $\{\textsl{Job}, \textsl{Loc(action)}, \textsl{Inc(ome)}\}$.
In this example, the \textsl{Job} feature $\in x_P$ is public, whereas \textsl{Loc} and \textsl{Inc} $\in x_S$ are sensitive.
We also consider a trained linear model \(f_\theta = 1.0\,\textsl{Job} - 0.5\,\textsl{Loc} + 0.5\,\textsl{Inc} \geq 0\)  and look at a scenario where user (A) has a public feature $\textsl{Job}=1.0$, and user (B) has a public feature $\textsl{Job}=-0.9$.
Both users' sensitive feature values are unknown. However, for user A, the outcome can be conclusively determined without revealing any additional information since all features are bounded within $[-1,1]$.
In contrast, for user B, the outcome cannot be determined solely based on the public feature, but revealing the sensitive feature $\textsl{Loc} = 1.0$ is enough to confirm the classifier outcome.

This example highlights two important observations that motivate our study: {\bf (1)} \hl{not all sensitive attributes may be required for decision-making during inference}, and {\bf (2)}
\hl{different individuals may need to disclose different amounts and types of sensitive information for decision-making}.

\vspace{-4pt}
\section{Core feature sets} 
\label{sec:core_feature_sets}
\vspace{-4pt}

With these considerations in mind, this section introduces the notion of \emph{core feature set}, the \emph{first contribution of the paper}, 
which will be used to quantify data minimization. 
The paper presents the key findings and defers all proofs in Appendix \ref{app:proofs}.

Throughout the paper, the symbols $R$ and $U$ are used to represent the sets of indices for revealed and unrevealed features of the sensitive attribute $S$, respectively. Given a vector $x$ and an index set $I$, we use $x_I$ to denote the vector containing entries indexed by $I$ and $X_I$ to represent the corresponding random variable. Finally, we write $f_{\theta}(X_U, X_R\!=\!x_R)$ as a shorthand for $f_\theta(X_U, X_R\!=\!x_R, X_P\!=\!x_P)$ to denote the prediction made by the model when the features in $U$ are unrevealed. 

\hl{Our objective is to develop algorithms that can identify the smallest subset of sensitive features to disclose}, ensuring that the model's output is accurate (with high probability) irrespective of the values of the undisclosed features. We refer to this subset as the \emph{core feature set}.

\begin{definition}[Core feature set]
\label{def:core_set}
Consider a subset $R$ of sensitive features $S$, and let $U \!=\! S \setminus R$ be the unrevealed features. The set $R$ is a core feature set if, for some $\tilde y \in \cY$,
\begin{equation}\label{eq:cfs}
    \Pr\big( f_\theta(X_{U}, X_{R} = x_R) = \tilde y \big) \geq 1-\delta,
\end{equation}
where $\delta \in [0,1]$ is a failure probability. 
\end{definition}
When $\delta = 0$ the core feature set is called \textbf{pure}. 
Additionally, the label $\tilde y$ satisfying Equation \eqref{eq:cfs} 
is called the \textbf{representative label} for the core feature set $R$. 
The concept of the representative label $\tilde y$ is crucial for the algorithms that will be discussed later. These algorithms use limited information to make predictions and when predictions are made using a set of unrevealed features, the representative label $\tilde{y}$ will be used in place of the model's prediction. 

In identifying core feature sets to minimize data leakage, it's crucial to consider model uncertainty, which refers to the unknown values of unrevealed features. The following result links core feature sets with model entropy, which measures uncertainty and is used by this work to minimize data leakage. 
\begin{proposition}
    \label{thm:delta_vs_entropy}
    Let $R \subseteq S$ be a core feature set with failure probability $\delta <0.5$. 
    Then, there exists a monotonic decreasing function $\epsilon:\mathbb{R}_+\to\mathbb{R}_+$ with $\epsilon(1)=0$ such that: 
  \[
      H\big[ f_\theta( X_U, X_R = x_R ) \big] \leq \epsilon(1-\delta),
  \]
where $H[Z] \text{=} -\sum_{z \in [L]}\Pr(Z=z) \log \Pr(Z=z)$ 
is the entropy of the random variable $Z$. 
\end{proposition}

This property \hl{highlights the relationship between core feature sets and entropy associated with a model using incomplete information}. As the $\delta$ value decreases, the model's predictions become more certain. When $\delta$ equals zero (or when $R$ represents a pure core feature set), the model's predictions can be fully understood without observing $x_{U}$, resulting in entropy of $0$. 

It is worth noticing that enhancing prediction accuracy necessitates revealing additional information, as illustrated by the previous result and the renowned information theoretical proposition below:
\begin{proposition}
    \label{thm:cond_entropy}
    Given two subsets $R$ and $R'$ of sensitive features $S$, with 
    $R \subseteq R'$, 
    \[
      H\big( f_\theta(X_U, X_R=x_R) \big) \geq 
      H\big( f_\theta(X_{U'}, X_{R'}=x_{R'}) \big),
    \]
    where $U=S\setminus R$ and $U'=S\setminus R'$.
\end{proposition}

\begin{wrapfigure}[9]{r}{150pt}
\vspace{-36pt}
\centering
    \includegraphics[width=\linewidth]{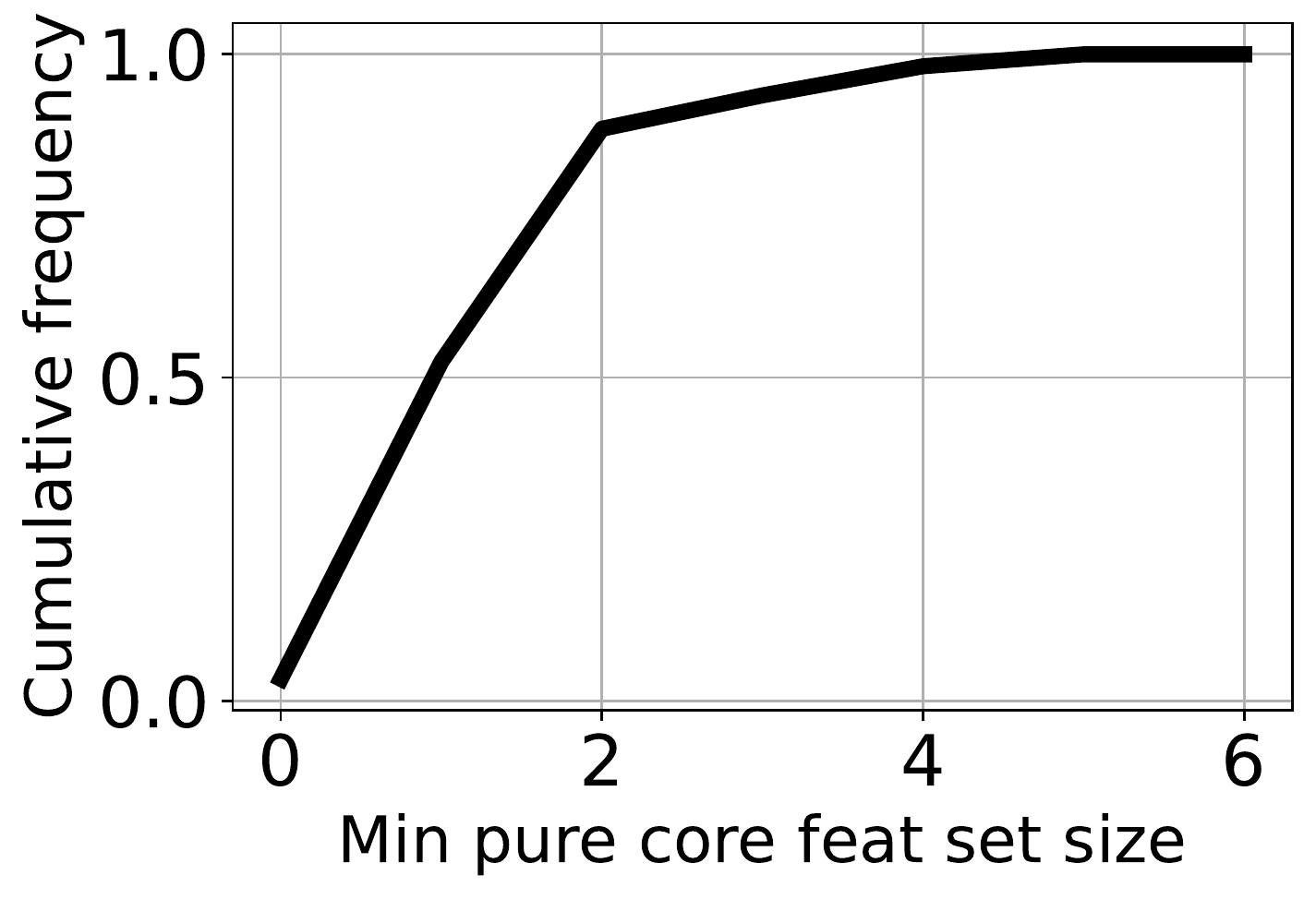} 
    \vspace{-18pt}
    \caption{\small Frequency associated with the size of the \textbf{minimum} pure core feature set.}
    \label{fig:motivation_C}
\end{wrapfigure}
Thus, the parameter $\delta$ plays a crucial role in balancing the trade-off between \emph{privacy loss} and \emph{model performance}. It determines the amount of sensitive information that must be disclosed to make accurate predictions, given a desired level of uncertainty in the model's predictions. As $\delta$ increases, fewer sensitive features need to be revealed, resulting in reduced data leakage but also less precise model predictions, and vice versa.

As highlighted in the previous example, the core feature set is not uniform for all users. This is further exemplified in Figure \ref{fig:motivation_C}, using the Credit dataset \cite{UCIdatasets} with a logistic regression classifier. The figure reports the cumulative count of users against the minimum number of features they need to disclose to ensure confident predictions. It demonstrates that many individuals need to disclose \emph{no} additional information to attain accurate predictions (corresponding to a pure feature of set size $0$), and most individuals can achieve accurate predictions by disclosing only $\leq 2$ sensitive features. These insights, together with the previous observations linking core feature sets to entropy, motivate the proposed online algorithm, the second contribution of the paper.

\section{MinDRel: An algorithm to minimize data release at inference time}
\label{sec:MinDRel}

The goal of the proposed algorithm, called \emph{Minimize Data Reveal} (MinDRel), is to uphold privacy during inference by revealing sensitive features one at a time based on their \emph{released} feature values. This section provides a high-level description of the algorithm and outlines its challenges. 
{Next, Section \ref{sec:MinDRel_linear}, applies MinDRel to 
linear classifiers and discusses its performance on several datasets 
and benchmarks. Further, Section \ref{sec:MinDRel_nonlinear}, extends MinDRel to non-linear classifiers and considers an evaluation over a range of standard datasets.} 

\textbf{Overview of MinDRel.}~
MinDRel operates fundamentally on two critical actions: \bhl{1.}~\hl{determining the next feature to reveal} for a particular user and \bhl{2.}~\hl{verifying whether the disclosed features make up a core feature set} for that user. These two operations will be discussed in sections \ref{sec:computingF} and \ref{sec:testing_CFS}, respectively. 

The algorithm determines which feature to disclose for a specific user by inspecting the posterior probabilities $\Pr(X_j | X_R = x_R, X_P = x_P)$ for each unrevealed feature $j \in U$, taking into account the disclosed sensitive features $x_R$ and public features $x_P$. Given the current set of disclosed features $x_R$ and unrevealed features $x_U$, MinDRel chooses the subsequent feature $j \in U$ as follows:
\begin{align}
  j &= \argmax_{j \in U} F(x_R, x_j; \theta) 
  \label{eq:scoring}
    \doteq \argmax_{j \in U} -H\big[ f_\theta(X_j=x_j, X_{U\setminus \{j\}}, X_R=x_R)\big],
\end{align}
where $F$ is a \emph{scoring function} that evaluates the amount of information that can be acquired about the model's predictions when feature $X_j$ is disclosed. As suggested in previous sections, it's desirable to reveal the feature that provides the most insight into the model prediction upon disclosure. MinDRel uses \emph{Shannon entropy} for this purpose as it offers a natural method for quantifying information. Once feature $X_j$ is disclosed with a value of $x_j$, the algorithm updates the posterior probabilities for all remaining unrevealed features. The process concludes either when all sensitive features have been disclosed or when a core feature set has been identified. It should also be noted that, within this framework, there is no need to perform data imputation when some features are missing. Unrevealed features are treated as random variables and are integrated during the prediction process.

Both the computation of the scoring function $F$ and the verification of whether a set of disclosed features constitute a core feature set present two significant challenges for the algorithm. The rest of the section delves into these difficulties.

\subsection{Computing the scoring function $F$}
\label{sec:computingF}
Designing a scoring function $F$ that measures how confident a model's prediction is when a user discloses an additional feature $X_j$ brings up two key challenges. {\bf First}, \emph{the value of $X_j$ is unknown until the decision to reveal it is made}, which complicates the computation of the entropy function. {\bf Second}, even if the value of $X_j$ were known, \emph{determining the entropy of model predictions in an efficient manner} is another difficulty.
We next discuss how to overcome these challenges.

\textbf{Dealing with unknown values.}
To address the first challenge, we exploit the information encoded in the disclosed features to infer $X_j$'s value and compute the posterior probability $\Pr(X_j | X_R \!=\! x_R)$ of the unrevealed feature $X_j$ given the values of the revealed ones. The scoring function, abbreviated as $F(X_j)$, can thus be modeled as the expected negative entropy given the randomness of $X_j$,
\begin{align}
\label{eq:exp_entropy}
    F(X_j) &= 
    \mathbb{E}_{X_j} - \big[H[f_\theta(X_j, X_{U\setminus \{j\}}, X_R \!=\! x_R)\big] \notag\\
    &= - \int \underbrace{
    H\big[\red{f_\theta (X_j\!=\!z, X_{U\setminus \{j\}}, X_R\!=\!x_R)}\big]
    }_{\red{A}}
    \times 
    \underbrace{\blue{\Pr(X_j \!=\! z | X_R \!=\! x_R)}}_{\blue{B}} dz,
\end{align}
where $z \in {\cal X}_j$ is a value in the support of $X_j$. 

\textbf{Efficient entropy computation.} 
The second difficulty relates to how to estimate this scoring function efficiently. Indeed this is challenged by two key components. The first (\red{A}) is the entropy of the model's prediction given a specific unrevealed feature value, $X_j = z$. This prediction is a function of the random variable $X_{U \setminus \{j\}}$, and, due to Proposition \ref{thm:delta_vs_entropy}, its estimation is linked to the conditional densities \red{$\Pr(X_{U \setminus \{j\}} \vert X_R = x_R, X_j = z)$}. The second (\blue{B}) is the conditional probability \blue{$Pr(X_j = z| X_R = x_R)$}. Efficient computation of these conditional densities is discussed next.

First, we discuss a result relying on the joint Gaussian assumption of the input features. This result will be useful in providing a computationally efficient method to estimate such conditional density functions. In the following, $\Sigma_{IJ}$ represents a sub-matrix of size $|I| \times |J|$ of a matrix $\Sigma$ formed by selecting rows indexed by $I$ and columns indexed by $J$. 

\begin{proposition}
\label{prop:2} The conditional distribution of any subset of unrevealed features 
$U' \subseteq U$, given the the values of released features $X_R =x_R$ is given by:
\begin{align*}
\blue{\Pr(X_{U'} | X_R = x_R)}  &= 
    \mathcal{N}\bigg(\mu_{U'}  + \Sigma_{U' R} \Sigma^{-1}_{R R} 
        (x_R - \mu_R), 
         \Sigma_{U' U'} - \Sigma_{U' R}\Sigma^{-1}_{R R} \Sigma_{R U'} \bigg),
\end{align*}
where $\Sigma$ is the covariance matrix.
\end{proposition}
Note that Equation \eqref{eq:exp_entropy} considers $U'\!=\! \{j\}$, and thus, component (\blue{B}) can be computed efficiently exploiting the result above. 
To complete Equation \eqref{eq:exp_entropy}, we need to estimate the entropy $H[\red{f_\theta (X_j = z, X_{U\setminus {j}}, X_R\!=\!x_R)}]$ (component \red{A}) for a specific instance $z$ drawn from $\blue{\Pr(X_j | X_R \!=\! x_R)}$. This poses a challenge due to the non-linearity of the hard model predictions $f_\theta$ adopted. To tackle this computational challenge, we first estimate component A using soft labels $\tilde{f}_\theta$ and then apply a thresholding operator. More specifically, we first estimate $\Pr(\tilde{f}_{\theta}(X_j \!=\! z,$ $X_{U\setminus\{j\}}, X_R \!=\! x_R))$ and, based on this distribution, we subsequentially estimate $f_\theta$ as $\bm{1} \{\tilde{f}_\theta \geq 0\}$, where $\bm{1}$ is the indicator function. In the following sections, we will show how to assess this estimate for linear and non-linear classifiers. Finally, by approximating the distribution over soft model predictions through Monte Carlo sampling, the score function in $F(X_j)$ can be computed as
\begin{align}
     F(X_j) 
     \approx -  \nicefrac{1}{| {\bm Z}|} \!\sum_{z' \in {\bm Z}}\!
     H\left[ \red{f_\theta (X_j\!=z\!', X_{U\setminus\{j\}}, X_R \!=x\!_R)} \right],\label{eq:monte_carlo} 
\end{align}
where $\bm Z$ is a set of random samples drawn from \blue{$\Pr(X_j | X_R = x_R)$} and estimated through Proposition \ref{prop:2}, which thus can be computed efficiently. 

When the Gaussian assumption does not hold, one can recur to (slower) Bayesian approaches to estimate the uncertainty of unrevealed features  $X_U$ given the set of revealed features $X_R =x_R$. 
A common approach involves treating $X_U$ as the target variable and employing a neural network to establish the mapping $X_U = g_{w}(X_R)$. Utilizing Bayesian techniques \cite{krishnan2022bayesiantorch,NEURIPS2020_d3d94468}, the posterior of the network's parameter $p(w|D) = p(w) p(D | w)$ can be computed initially. Based on the posterior distribution of the model's parameters $w$, the posterior of unrevealed features can be calculated as $\Pr(X_U = x_U | X_R = x_R) = \Pr_{w \sim p( w | D)}(g_{w}(X_R) = x_U)$.
However, implementing such a Bayesian network not only significantly increases training time but also inference time. Since it is necessary to compute $\Pr(X_U |X_R = x_R)$ for all possible choices of $U \in S$, the number of Bayesian neural network regressors scales exponentially with $|S|$.

\begin{wrapfigure}[8]{r}{0.5\textwidth}
\vspace{-24pt}
\centering\includegraphics[width=\linewidth]{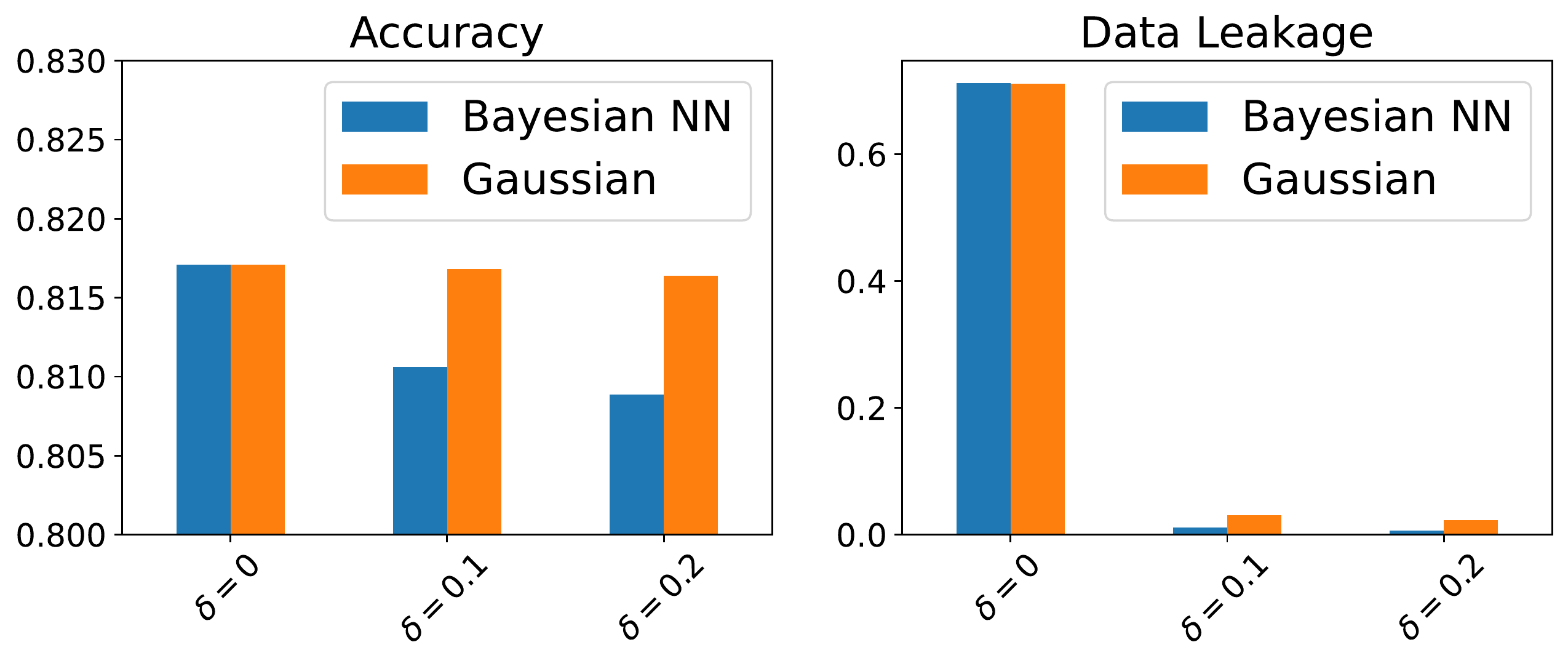}
\vspace{-18pt}
\caption{\small Bayesian NN vs Gaussian.}
\label{fig:compare_linear_gauss_bayes}
\end{wrapfigure}
Importantly, in our evaluation, the data minimization method that operates under the Gaussian assumption maintains similar decision-making and produces comparable outcomes to the Bayesian approach, even in cases where the Gaussian assumption is not applicable in practical settings. Figure 3 illustrates this comparison, showcasing the performance of the proposed mechanism on a real dataset (Credit dataset with $|S|=5$) concerning accuracy (higher is better) and data leakage (lower is better) across various failure probability $\delta$ values. Notice how similar is the performance of the mechanisms that either leverage the Gaussian assumption or operate without it (Bayesian NN). Importantly, the assumption of a Gaussian distribution is not overly restrictive or uncommon. In fact, it is a cornerstone in many areas of machine learning, including Gaussian Processes \cite{williams2006gaussian}, Bayesian optimization \cite{shahriari2015taking}, and Gaussian Graphical models \cite{peterson2015bayesian}. 

\subsection{Testing a core feature set}
\label{sec:testing_CFS}
The proposed iterative algorithm terminates once it determines whether a subset $R$ of the sensitive feature set $S$ constitutes a core feature set. This validation process falls into two scenarios:

\begin{enumerate}[leftmargin=*,itemsep=0pt,topsep=-4pt,partopsep=0pt]
\item {When $\delta\!=\!0$}: To confirm that $R$ is a pure core feature set, it is sufficient to verify that {$f_{\theta}(X_U, X_R=x_R)$} remains constant for all possible realizations of $X_U$. As we will show in Section \ref{sec:MinDRel_linear}, linear classifiers can perform this check in linear time without making any specific assumptions about the input distribution.

\item {When $\delta\!>\!0$}:  In this case, the property above is no longer valid. As per Definition \ref{def:core_set}, to confirm a core feature set, it is essential to estimate the distribution of $\Pr(\tilde{f}_{\theta}(X_U, X_R =x_R))$. In Section \ref{sec:MinDRel_linear}, we demonstrate how to analytically estimate this distribution for linear classifiers. Furthermore, in Section \ref{sec:MinDRel_nonlinear}, we illustrate how to locally approximate this distribution for nonlinear classifiers and derive a simple yet effective estimator that can be readily implemented in practice.
\end{enumerate}

\section{MinDRel for linear classifiers}
\label{sec:MinDRel_linear}

This section will  devote to estimating the distribution 
$\Pr(\red{\tilde{f}_{\theta}(X_j \!=\! z, X_{U\setminus\{j\}}, X_R \!=\! x_R)})$, 
simply expressed as $\Pr(\red{\tilde{f}_{\theta}(X_U, X_R \!=\! x_R)})$ and provides an instantiation of MinDRel for linear classifiers. 
In particular, it shows that 
both the estimation of the conditional distributions required to compute 
the scoring function $F(X_j)$ and the termination condition to test whether 
a set of revealed features is a core feature set, can be computed efficiently. This is an important property for the developed algorithms, which are designed to be online and interactive.

\vspace{-4pt}
\subsection{Efficiently Estimating $\Pr(\red{\tilde{f}_{\theta}(X_{U}, X_R = x_R)})$}
\vspace{-4pt}

For a linear classifier $\tilde{f}_{\theta} = \theta^\top x$, notice that when the input features are jointly Gaussian, the model predictions $\tilde{f}_{\theta}(x)$ are also Gaussian, as highlighted by the following result.
\begin{proposition}
\label{prop:3}
The model soft prediction, $\tilde{f}_{\theta}(X_U, X_R = x_R) =\theta_U X_U + \theta_R x_R$ 
is a random variable following a Gaussian distribution $\mathcal{N}\big(m_f, {\sigma_f^2} \big)$, 
with
\begin{align}
    m_f &= \theta_R x_R + \theta^\top_U \big(  \mu_U + \Sigma_{U R} \Sigma^{-1}_{RR} ( x_{R} - \mu_R) \big) \label{eq:prop4_a}\\ 
    \sigma^2_f &= \theta^\top_{U} \big( \Sigma_{U U} - \Sigma_{U R} \Sigma^{-1}_{R R} \Sigma_{R U} \big) \theta_{U}, \label{eq:prop4_b}
\end{align}
where $\theta_U$ is the sub-vector of parameters $\theta$ corresponding to the unrevealed features $U$. 
 \end{proposition} 

The result above is used to assist in calculating the conditional 
distribution of the model hard predictions $f_\theta(x)$, following thresholding. 
This is a random variable that adheres to a Bernoulli distribution, as shown 
next, and will be used to compute the entropy of the model predictions, 
as well as to determine if a subset of features constitutes a core set.
 
\begin{proposition}
\label{prop:4}
Let the soft model predictions $\tilde{f}_\theta(X_U, X_R =x_R)$ 
be a random variable following a Gaussian distribution $\mathcal{N} (m_f, 
\sigma^2_f)$. Then, the model prediction following thresholding $f_{\theta}( X_U, X_R =x_R)$ is a random variable following a Bernoulli distribution $\mathcal{B}(p)$ 
with $p = \Phi( \frac{m_f}{\sigma_f})$, where $\Phi(\cdot)$ is the  CDF of the standard Normal distribution, and $m_f$ and $\sigma_f$,
are given in Eqs~\eqref{eq:prop4_a} and \eqref{eq:prop4_b}, respectively.
\end{proposition}

\subsection{Testing pure core feature sets}
\label{sec:pure_core_set_linear}

In this subsection, we outline the methods for determining if a subset $U$ is a pure core feature set, and, if so, identifying its representative label. As per Definition \ref{def:core_set}, $U$ is a pure core feature set if $f_{\theta}(X_U, X_R =x_R) = \tilde{y}$ for all $X_U$. This implies that $\tilde{f}_{\theta}(X_U, X_R = x_R) = \theta^\top_U X_U + \theta^\top_R x_R$ must have the same sign for all $X_U$ in the range of $[-1, 1]^{|U|}$. Given the box constraint $X_U \in [-1, 1]^{|U|}$, the linearity of the model considered allows us to directly compute the maximum and minimum values of $\tilde{f}_{\theta}(X_U, X_R = x_R)$, rather than enumerating all possible values. Specifically, we have:
\begin{align*}
\max_{X_U} \theta^\top_U X_U + \theta^\top_R x_R &= \|\theta_U\|_1 + \theta^\top_R x_R\\
\min_{X_U} \theta^\top_U X_U + \theta^\top_R x_R &= - \|\theta_U\|_1 + \theta^\top_R x_R.
\label{eq:min_max_linear}
\end{align*}
Thus, if both these maximum and minimum values are negative (non-negative), then $U$ is considered a pure core feature set with representative label $\tilde{y}=0$ ($\tilde{y}=1$). If not, $U$ is not a pure core feature set.

Importantly, determining whether a subset $R$ of sensitive features $S$ constitutes a pure core feature set can be accomplished in linear time with respect to the  number of features. 
\begin{proposition}
\label{thm:core_linear_verf}
Assume $f_\theta$ is a linear classifier. Then, determining if a 
subset $U$ of sensitive features $S$ is a \emph{pure} core feature set 
can be performed in $O(|P| + |S|)$ time.
\end{proposition}

\subsection{MinDRel-linear Algorithm and Evaluation}
\label{sec:MinDRellinear_eval}

A pseudo-code of MinDRel specialized for linear classifiers is reported in 
Algorithm \ref{alg:alg1}. 
At inference time, the algorithm takes as input a sample $x$ 
(which only exposes the set of public features $x_P$) and uses the training data
$D$ to estimate the mean and covariance matrix needed to compute the conditio-

\begin{wrapfigure}[19]{R}{0.58\textwidth}
\vspace{-2pt}\hspace*{4pt}
    \begin{minipage}{0.56\textwidth}
    \begin{algorithm}[H]
    \footnotesize
  \caption{MinDRel for linear classifiers}
  \label{alg:alg1}
  \setcounter{AlgoLine}{0}
  \SetKwInOut{Input}{input}
  \SetKwInOut{Output}{output}
  \SetKwRepeat{Do}{do}{while}
  \Input{A test sample $x$; training data $D$} 
  \Output{A core feature set $R$  and its representative label $\tilde y$\!\!\!\!\!\!\!}
  $(\mu, \Sigma)  \gets \left( \frac{1}{|D|} \sum_{(x,y) \in D} x, 
  \frac{1}{|D|} \sum_{(x,y) \in D} (x -\mu) (x - \mu)^\top\right)$\!\!\!\!\!\!\!\!\!\!\!\!\!\!\!\!\\ 
  Initialize $R = \emptyset$\\
  \While{True} {
    \eIf{$R$ is a core feature set with repr.~label $\tilde y$}{
      \Return{$(R, \tilde y)$}
    }{
    \ForEach{$j \in U$} {
       Compute $\blue{\Pr(X_j | X_R = x_R)}$ (using Prop.~\ref{prop:2})\\
      $\bm Z \gets \text{sample}(\blue{\Pr(X_j |X_R = x_R)})$ T times\\
       Compute $\Pr\left(\red{f_{\theta}(X_j = z, X_{U\setminus \{j\}}
        X_R = x_R)}\right)$\!\!\!\!\!\!\!\!\!\!
       (using Prop.~\ref{prop:3} and \ref{prop:4})\\
       Compute $F(X_j)$ (using Eq. \eqref{eq:monte_carlo})
       }
    }
     $j^* \gets \argmax_j F(X_j)$\\
     $(R, U) \gets R \cup \{j^*\},\; U \setminus \{j^*\}$
    }
    \vspace{-2pt}
    \end{algorithm}
    \end{minipage}
\end{wrapfigure}
nal distribution of the model predictions given the unrevealed features (line 1), 
as discussed above. 
After initializing empty the set of revealed features (line 2),
it iteratively releases a feature at a time until a core feature set 
(and its associated representative label) are determined, as
discussed in detail in Section \ref{sec:pure_core_set_linear}.
The released feature $X_{j^*}$ is the one, among the unrevealed features 
$U$, that maximizes the scoring function $F$ (line 12). Computing such 
a scoring function entails estimating the conditional distribution 
$\Pr(X_j | X_R = x_R)$ (line 8), constructing a sample set $\bm Z$ 
from such distribution (line 9), and approximating the distribution 
over soft model predictions through Monte Carlo sampling to compute (line 10).
Finally, after each iteration, the algorithm updates the set of the 
revealed and unrevealed features (line 13). 

Notice that MinDRel relies on estimating the mean vector and covariance matrix 
from the training data, which is considered public, for the scope of this paper. 
If the training data is private, various techniques exist to release DP mean, 
and variance \cite{liu2021robust,amin2019differentially} and can be readily adopted. 
Nonetheless, the protection of training data is beyond the focus of this work.

\textbf{Evaluation.} 
Next, this section evaluates MinDRel's effectiveness  in limiting data exposure 
during inference. The experiments are conducted on six standard UCI datasets \cite{UCIdatasets} 
Here, we present a selection of the results and discuss their trends. A comprehensive overview, additional analysis, and experiments are available in the Appendix.

\begin{wrapfigure}[11]{r}{0.4\textwidth}
\vspace{-22pt} 
\centering\includegraphics[width=\linewidth]{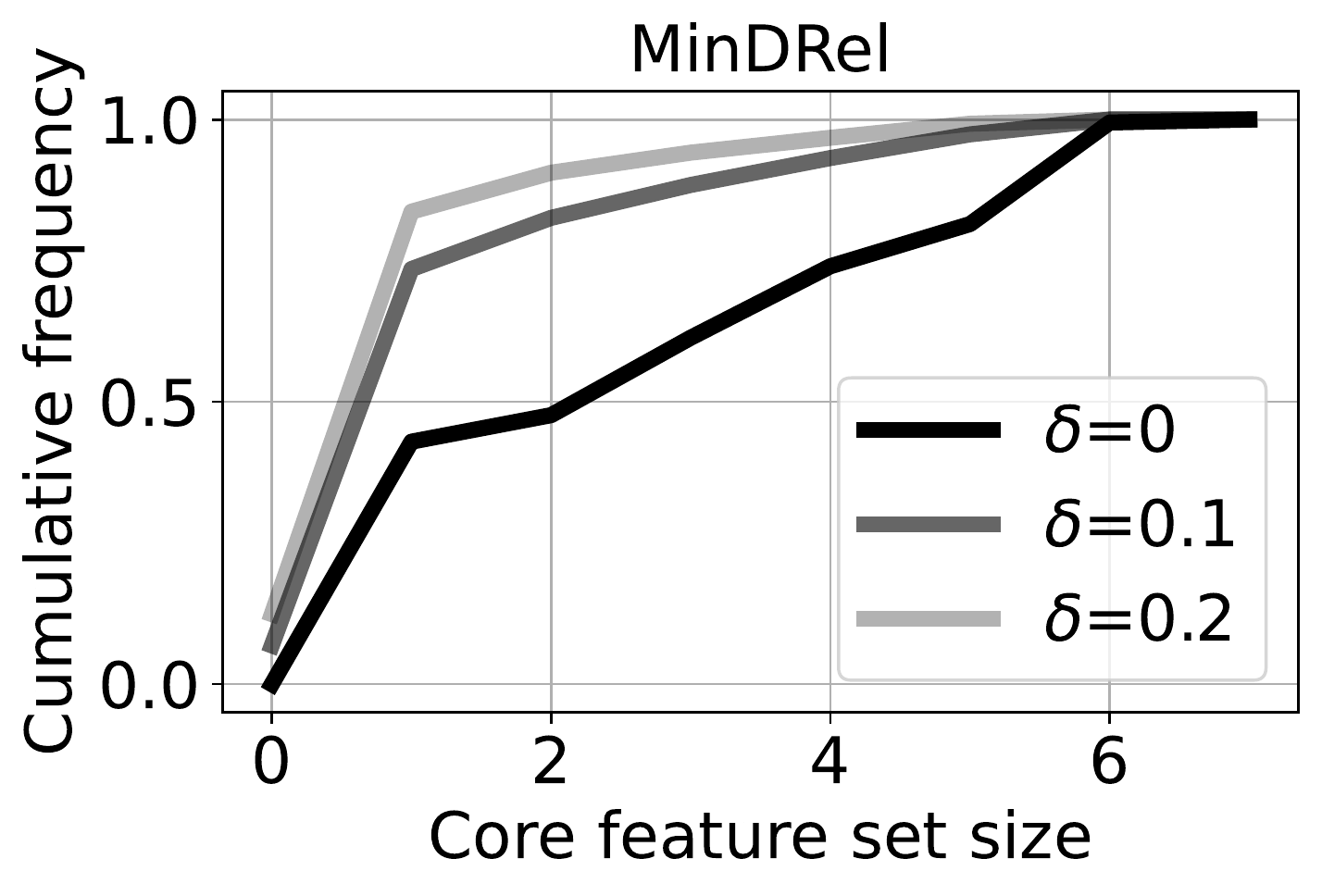}
\vspace{-20pt}
\caption{\small Cumulative count of users against their core feature set size for various $\delta$.}
\label{fig:hist_linear}
\end{wrapfigure}
Figure \ref{fig:hist_linear} reports the cumulative count of users (y-axis) 
against the minimum number of features they need to disclose (x-axis) to ensure 
confident predictions, for various failure probabilities $\delta$.
The model adopted is a Logistic Regression classifier trained on the Bank dataset \cite{UCIdatasets}. 
The data minimization achieved by MinDRel is clearly apparent. 
For each test sample, MinDRel identifies core feature sets that are 
significantly smaller than the total sensitive feature set size $|S|=7$. 
Interestingly, when $\delta >0$, it discovers core feature sets 
smaller than 2 for the majority of users. \hl{This implies that most 
users would only need to reveal a small portion of their sensitive 
data to achieve accurate model predictions with either absolute certainty or high confidence}.

To further emphasize the benefits of MinDRel, we compare it to two baselines: the \emph{optimal} model, which employs a brute force method to find the smallest core feature set and its representative label and assumes all sensitive features are known, and the \emph{all-features} model, which simply adopts the original classifiers using all the data features for each test sample. 
The performances are displayed for a varying number of sensitive attributes $|S| \in [2,5]$, while we delegate a study for larger $|S|$ to the Appendix (which excludes the time-consuming baseline \emph{optimal}, as intractable for large $|S|$). 
For each choice of $|S|$, we randomly select $|S|$ features from the entire feature set and designate
them as sensitive attributes. The remaining attributes are considered public. The average
accuracy and data leakage are then reported based on 100 random sensitive attribute selections. More details about the experimental settings are in Appendix Section \ref{sec:app_exp}.

\vspace{-2pt}
Figure \ref{fig:linear} depicts performance results in accuracy (top subplots/higher is better) and relative data leakage (bottom subplots/lower is better) with varying number of revealable sensitive features $|S|$. 
The comparison includes three MinDRel versions: \textsl{F-Score}, which utilizes the scoring function elaborated in Section \ref{sec:computingF} to select the next feature to disclose (left); \textsl{Importance}, which employs a feature importance criterion leveraging the model $f_\theta$ parameters to rank features, detailed in the

\begin{wrapfigure}[16]{r}{0.6\linewidth}
\vspace{-4pt}
\centering\includegraphics[width=\linewidth]{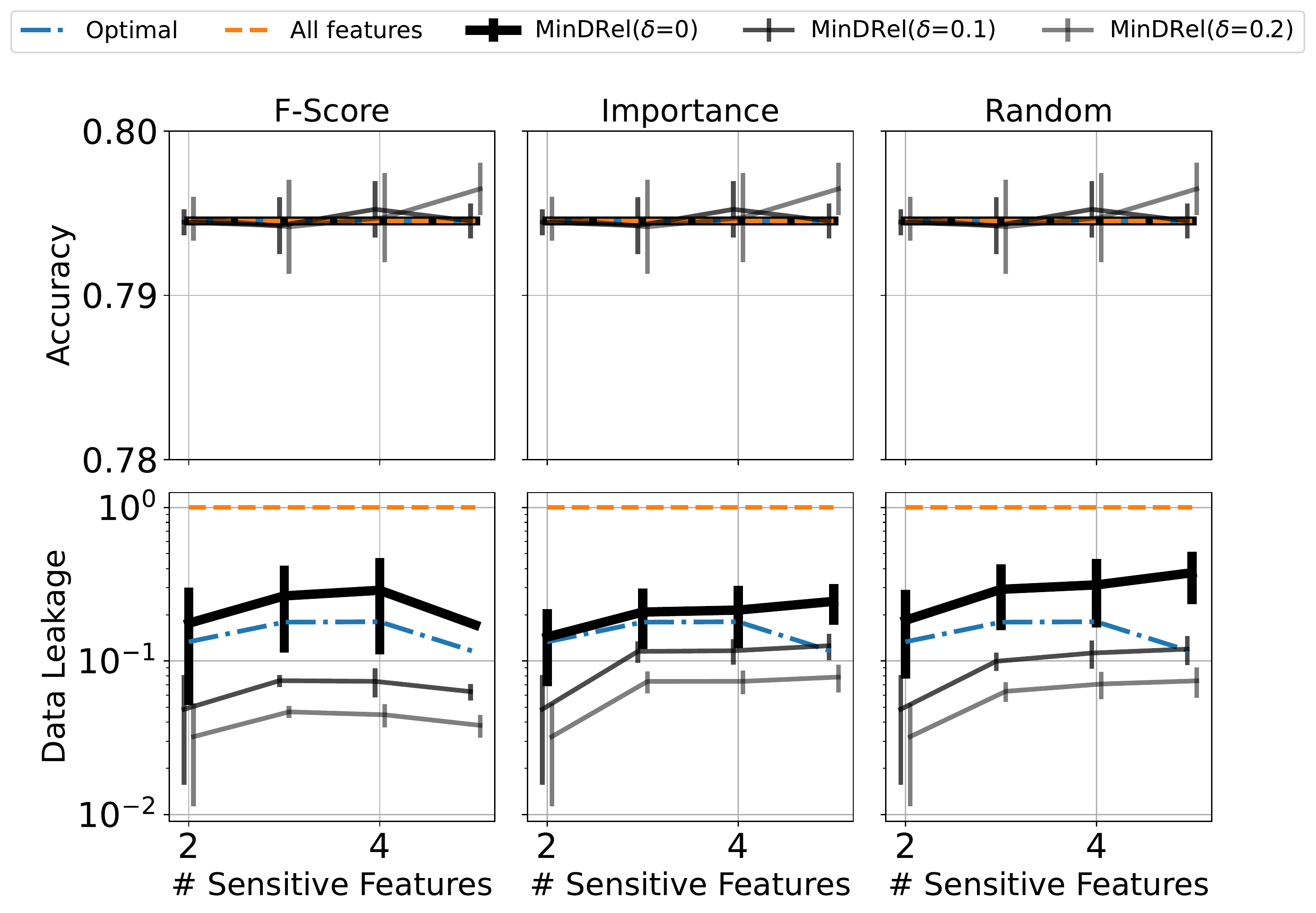}
\vspace{-18pt}
\caption{\small Logistic regression classifiers.}
\label{fig:linear}
\end{wrapfigure}
Appendix (middle); and \textsl{Random}, which arbitrarily selects the next feature to reveal (right). All three versions use the same test procedure to validate whether a set qualifies as a core feature set.
Firstly, notice that MinDRel achieves equal or better accuracy than the optimal mechanism and baseline that utilizes all features during testing. \hl{This suggests that data minimization can be accomplished under linear models without compromising accuracy!}
Next, observe that an increase in $\delta$ aids in safeguarding data minimization, as evidenced by the drop in relative data leakage (note the logarithmic scale of the y-axis). This is attributable to the influence of $\delta$ on the test for identifying a core feature set, thereby reducing its size.
Finally, the proposed scoring function outperforms the other versions in terms of data leakage minimization, allowing users to disclose substantially fewer sensitive features.
The Appendix also includes experiments with larger quantities of sensitive features, presenting analogous trends, where, however, a comparison against an optimal baseline was not possible in due to its exponential time complexity.

\section{MinDRel for non-linear classifiers}
\label{sec:MinDRel_nonlinear}

Next, the paper focuses on computing the estimate $\Pr(\red{\tilde{f}_\theta(X_U , X_R = x_R)})$  and determining core feature sets when $f_\theta$ is a nonlinear classifier. Then, the section presents results that illustrate the practical benefits of  MinDRel for data minimization at inference time on neural networks.

\subsection{Efficiently estimating $Pr(\red{\tilde{f}_{\theta}(X_U, X_R =x_R)})$}

First notice that even if the input features $x$ are jointly Gaussian, the outputs $f_\theta(x)$ will no longer adhere to a Gaussian distribution after a non-linear transformation. This complicates estimating the distribution $\Pr(\tilde{f}_{\theta}(X_U, X_R \!=\!x_R)$. To tackle this challenge, the paper takes a local approximation of the model predictions $\tilde{f}_\theta$ using a Gaussian distribution, as demonstrated in the following result.
\begin{theorem}
\label{thm:taylor_approx}
The distribution of the  random variable $\tilde{f}_\theta = \tilde{f}_{\theta}(X_U, X_R =x_R)$ where $X_U \sim \mathcal{N}\big(\mu^{\mbox{pos}}_{U}, \Sigma^{\mbox{pos}}_{U} \big)$ can be  approximated by a Normal distribution as
\begin{equation}\label{eq:taylor_arppox}
    \tilde{f}_\theta \sim \mathcal{N} \left( 
    \red{\tilde{f}_{\theta}( X_U = \mu^{pos}_U, X_R = x_R)}, \, g^\top_{U}\Sigma^{pos}_{U}g_U\right),
\end{equation}
where $g_{U} = \nabla_{X_U}\tilde{f}_{\theta}(  X_U = \mu^{pos}_{U}, X_R = x_R)$ is the gradient of model prediction at $X_U = \mu^{pos}_{U}$. 
\end{theorem}
Therein, the mean vector $\mu^{\mbox{pos}}_{U}$ and covariance matrix $\Sigma^{\mbox{pos}}_{U}$ of $Pr(X_U |X_R \!=\! x_R)$ are derived from Proposition \ref{prop:2}. This result leverages a first-order Taylor approximation of model $f_\theta$ around its mean.

\subsection{Testing pure core feature sets}
\label{sec:pure_core_non_linear}

Unlike linear classifiers, the case for nonlinear models lacks an exact and efficient method to determine whether a set is a core feature set. This is primarily due to the non-convex nature of the adopted models, which poses challenges in finding a global optimum. This section thus proposes an approximate testing routine and demonstrates its practical ability to significantly minimize data leakage during testing while maintaining high accuracy.

To determine if a subset $U$ of the sensitive features $S$ is a pure core feature set, we consider a set $Q$  of 
input points $(X_U, x_R)$. The entries corresponding to the revealed features are set to the value $x_R$, while the entries corresponding to the unrevealed features are sampled from the distribution $\Pr(X_U | X_R = x_R)$. 
The test verifies if the model predictions $f_{\theta}(x)$ remain constant for all $x$ in $Q$. 
In the next section, we will show that even considering arbitrary classifiers (e.g., we use standard neural networks), MinDRel can reduce data leakage dramatically when compared to standard approaches. 

\subsection{MinDRel-nonlinear Algorithm and Evaluation}
The MinDRel algorithm for nonlinear classifiers differs from Algorithm \ref{alg:alg1} primarily in the method used to compute the estimates for the distribution $\Pr(f_\theta(X_j = z, X_{U\setminus {j}}, X_R=x_R))$ of the soft model predictions  (line 11). 
In this case, this estimate is computed by leveraging the results from Theorem \ref{thm:taylor_approx} and Proposition \ref{prop:4}. Moreover, the termination test of the algorithm is based on the discussion in the previous section. Appendix \ref{app:pseudocodes} reports a description of the algorithm's pseudocode.

\textbf{Evaluation.}
To assess MinDRel's performance in reducing data leakage when employing standard nonlinear classifiers, we use a neural network with two hidden layers and ReLU activation functions and train the models using stochastic gradient descent (see Appendix \ref{sec:app_exp} for additional details). The evaluation criteria, baselines, and benchmarks adhere to the same parameters set in Section \ref{sec:MinDRellinear_eval}.

 \begin{wrapfigure}[16]{r}{0.6\linewidth}
\vspace{-18pt}
\centering\includegraphics[width=\linewidth]{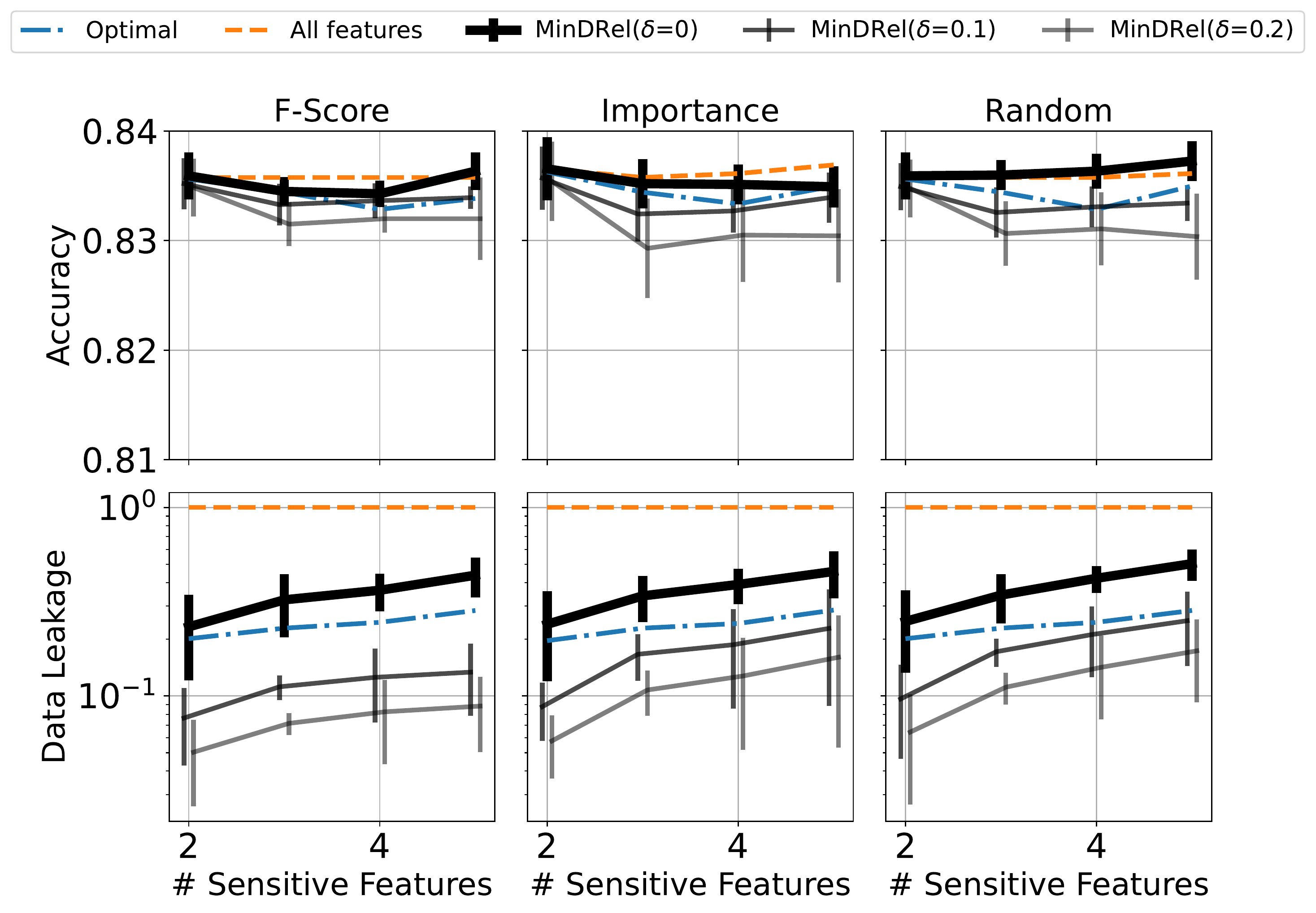}
\vspace{-18pt}
\caption{\small Neural network classifiers.}
\label{fig:nonlinear}
\end{wrapfigure}
Figure \ref{fig:nonlinear} showcases the results in terms of accuracy (top subplots) and data leakage (bottom subplots). Unlike linear classifiers, nonlinear models using MinDRel with a failure probability of $\delta=0$ cannot guarantee the same level of accuracy as the ``\textsl{all features}'' baseline model. However, this accuracy difference is minimal. Notably, a failure probability of $\delta=0$ enables users to disclose less than half, and up to 90\% fewer sensitive features across different datasets, while achieving accuracies similar to those of conventional classifiers. Next, similarly to as observed in the previous section, MinDRel with the proposed F-score selector significantly outperforms the other variants in terms of data leakage minimization. 
Furthermore, when considering higher failure probabilities, data leakage decreases significantly. For instance, with $\delta \leq 0.1$, users need to disclose only 5\% of their sensitive features while maintaining accuracies comparable to the baseline models (the largest accuracy difference reported was 0.005\%). 
These results are significant: \hl{They show that the introduced notion of privacy leakage and the proposed algorithm can become a valuable tool to protect individuals' data privacy at test time, without significantly compromising accuracy.}.

\section{Related work}
\label{app:related_work}

While we are not aware of studies on data minimization for inference problems, we draw connections with differential privacy, feature selection, and active learning.

\textbf{Differential Privacy.}
Differential Privacy (DP) \cite{Dwork:06} is a strong privacy notion that determines and bounds the risk of disclosing sensitive information of individuals participating in a computation. 
In the context of machine learning, DP ensures that algorithms can learn the relations between data and predictions while preventing them from memorizing sensitive information about any specific individual in the training data. 
In such a context, DP is primarily adopted to protect training data \cite{abadi:16, JMLR:v12:chaudhuri11a,xie2018differentially} and thus the setting contrasts with that studied in this work, which focuses on identifying the superfluous features revealed by users at \emph{test time} to attain high accuracy. 
Furthermore, achieving tight constraints in differential privacy often comes at the cost of sacrificing accuracy, while the proposed privacy framework can reduce privacy loss without sacrificing accuracy under the assumption of linear classifiers.

\textbf{Feature selection.} 
Feature selection \cite{chandrashekar2014survey} is the process of identifying and selecting a relevant subset of features from a larger set for use in model construction, with the goal of improving performance by reducing the complexity and dimensionality of the data. 
The problem studied in this work can be considered a specialized form of feature selection with the added consideration of personalized levels, where each individual may use a different subset of features. This contrasts standard feature selection \cite{li2017feature}, which selects the same subset of features for each data sample.
Additionally, and unlike traditional feature selection, which is performed during training and independent of the deployed classifier \cite{chandrashekar2014survey}, the proposed framework performs feature selection at deployment time and is inherently dependent on the deployed classifier.

\textbf{Active learning.}
Finally. the proposed framework shares similarities with active learning \cite{fu2013survey,settles2009active}, whose goal is to iteratively select samples for experts to label in order to construct an accurate classifier with the least number of labeled samples. Similarly, the proposed framework iteratively asks individuals to reveal one attribute given their released features so far, with the goal of minimizing the uncertainty in model predictions.

Despite these similarities, the proposed data minimization for inference concept is motivated by a privacy need and pertains to the analysis of features to release to induce the same level of accuracy as if all features were released.

\section{Conclusion and Future Work}
This paper introduced the concept of data minimization at test time whose goal is to minimize the number of features that individuals need to disclose during model inference while maintaining accurate predictions from the model. The motivations of this notion are grounded in the privacy risks imposed by the adoption of learning models in consequential domains, by the significant efforts required by organizations to verify the accuracy of the released information, and align with the data minimization principle outlined in the GDPR. 
The paper then discusses an iterative and personalized algorithm that selects the features each individual should release with the goal of minimizing data leakage while retaining exact (in the case of linear classifiers) or high (for non-linear classifiers) accuracy. 
Experiments over a range of benchmarks and datasets indicate that individuals may be able to release as little as 10\% of their information without compromising the accuracy of the model, providing a strong argument for the effectiveness of this approach in protecting privacy while preserving the accuracy of the model. 

\textbf{Avenues for future work.} 
While this study is the first attempt at defining data minimization during inference, it also opens up avenues for further research. First, establishing bounds on the data leakage provided by our proposed method compared to an optimal procedure presents an interesting and open challenge. Second, exploring the relationship between data minimization principles and their consequent disparate impacts presents another open direction. Lastly, developing effective algorithms to provably construct core feature sets for non-linear classifiers is another important area of investigation.

\section*{Acknowledgements}
This research was partially supported by NSF grant 2133169, 
a Google Research Scholar Award and an Amazon Research Award. 

\bibliography{lib.bib}

\begin{thebibliography}{10}

\bibitem{abadi:16}
Abadi and et~al.
\newblock Deep learning with differential privacy.
\newblock In {\em Proceedings of the 2016 ACM SIGSAC Conference on Computer and
  Communications Security}, 2016.

\bibitem{amin2019differentially}
K.~Amin, T.~Dick, A.~Kulesza, A.~Munoz, and S.~Vassilvitskii.
\newblock Differentially private covariance estimation.
\newblock {\em Advances in Neural Information Processing Systems}, 32, 2019.

\bibitem{bishop2006pattern}
C.~M. Bishop and N.~M. Nasrabadi.
\newblock {\em Pattern recognition and machine learning}, volume~4.
\newblock Springer, 2006.

\bibitem{UCIdatasets}
C.~Blake and C.~Merz.
\newblock Uci repository of machine learning databases, 1988.

\bibitem{chandrashekar2014survey}
G.~Chandrashekar and F.~Sahin.
\newblock A survey on feature selection methods.
\newblock {\em Computers \& Electrical Engineering}, 40(1):16--28, 2014.

\bibitem{JMLR:v12:chaudhuri11a}
K.~Chaudhuri, C.~Monteleoni, and A.~D. Sarwate.
\newblock Differentially private empirical risk minimization.
\newblock {\em Journal of Machine Learning Research}, 2011.

\bibitem{Dwork:06}
C.~Dwork, F.~McSherry, K.~Nissim, and A.~Smith.
\newblock Calibrating noise to sensitivity in private data analysis.
\newblock In {\em Theory of cryptography conference}, pages 265--284. Springer,
  2006.

\bibitem{fu2013survey}
Y.~Fu, X.~Zhu, and B.~Li.
\newblock A survey on instance selection for active learning.
\newblock {\em Knowledge and information systems}, 35(2):249--283, 2013.

\bibitem{krause2005note}
A.~Krause and C.~Guestrin.
\newblock {\em A note on the budgeted maximization of submodular functions}.
\newblock Citeseer, 2005.

\bibitem{krishnan2022bayesiantorch}
R.~Krishnan, P.~Esposito, and M.~Subedar.
\newblock Bayesian-torch: Bayesian neural network layers for uncertainty
  estimation, Jan. 2022.

\bibitem{NEURIPS2020_d3d94468}
R.~Krishnan and O.~Tickoo.
\newblock Improving model calibration with accuracy versus uncertainty
  optimization.
\newblock In {\em Advances in Neural Information Processing Systems},
  volume~33, pages 18237--18248, 2020.

\bibitem{fetal}
Larxel.
\newblock Children fetal health dataset.
\newblock
  \url{https://www.kaggle.com/datasets/andrewmvd/fetal-health-classification},
  2021.

\bibitem{li2017feature}
J.~Li, K.~Cheng, S.~Wang, F.~Morstatter, R.~P. Trevino, J.~Tang, and H.~Liu.
\newblock Feature selection: A data perspective.
\newblock {\em ACM computing surveys (CSUR)}, 50(6):1--45, 2017.

\bibitem{liu2021robust}
X.~Liu, W.~Kong, S.~Kakade, and S.~Oh.
\newblock Robust and differentially private mean estimation.
\newblock {\em Advances in neural information processing systems},
  34:3887--3901, 2021.

\bibitem{papernot2016towards}
N.~Papernot, P.~McDaniel, A.~Sinha, and M.~Wellman.
\newblock Towards the science of security and privacy in machine learning.
\newblock {\em arXiv preprint arXiv:1611.03814}, 2016.

\bibitem{peterson2015bayesian}
C.~Peterson, F.~C. Stingo, and M.~Vannucci.
\newblock Bayesian inference of multiple gaussian graphical models.
\newblock {\em Journal of the American Statistical Association},
  110(509):159--174, 2015.

\bibitem{rastegarpanah2021auditing}
B.~Rastegarpanah, K.~Gummadi, and M.~Crovella.
\newblock Auditing black-box prediction models for data minimization
  compliance.
\newblock {\em Advances in Neural Information Processing Systems},
  34:20621--20632, 2021.

\bibitem{regulation2016regulation}
P.~Regulation.
\newblock Regulation (eu) 2016/679 of the european parliament and of the
  council.
\newblock {\em Regulation (eu)}, 679:2016, 2016.

\bibitem{car_insurance}
S.~Roy.
\newblock Car insurance data.
\newblock \url{https://www.kaggle.com/datasets/sagnik1511/car-insurance-data},
  2021.

\bibitem{settles2009active}
B.~Settles.
\newblock Active learning literature survey.
\newblock 2009.

\bibitem{shahriari2015taking}
B.~Shahriari, K.~Swersky, Z.~Wang, R.~P. Adams, and N.~De~Freitas.
\newblock Taking the human out of the loop: A review of bayesian optimization.
\newblock {\em Proceedings of the IEEE}, 104(1):148--175, 2015.

\bibitem{customer}
A.~Sudarshan.
\newblock Customer segmentation data.
\newblock
  \url{https://www.kaggle.com/datasets/abisheksudarshan/customer-segmentation},
  2021.

\bibitem{williams2006gaussian}
C.~K. Williams and C.~E. Rasmussen.
\newblock {\em Gaussian processes for machine learning}, volume~2.
\newblock MIT press Cambridge, MA, 2006.

\bibitem{xie2018differentially}
L.~Xie, K.~Lin, S.~Wang, F.~Wang, and J.~Zhou.
\newblock Differentially private generative adversarial network.
\newblock {\em arXiv preprint arXiv:1802.06739}, 2018.

\end{thebibliography}
\bibliographystyle{abbrv}

\appendix
\setcounter{theorem}{0}
\setcounter{corollary}{0}
\setcounter{lemma}{0}
\setcounter{proposition}{0}

\section{Missing proofs}
\label{app:proofs}

\begin{proposition}
    \label{app:thm:delta_vs_entropy}
    Given a core feature set $R \subseteq S$ with failure probability $\delta <0.5$, then there exists a function $\epsilon:\mathbb{R}\to\mathbb{R}$ that is  monotonic decreasing function with $\epsilon(1)=0$ such that: 
  \[
      H\big[ f_\theta( X_U, X_R = x_R ) \big] \leq \epsilon(1- \delta),
  \]
where $H[Z] \text{=} -\sum_{z \in [L]}\Pr(Z=z) \log \Pr(Z=z)$ 
is the entropy of the random variable $Z$. 
\end{proposition}

\begin{proof}
In this proof, we demonstrate the binary classification case. The extension to a multi-class scenario can be achieved through a similar process.

By the definition of the core feature set, there exists a representative label, denoted as $\tilde{y} \in \{0,1\}$ such that the probability of $P(f_{\theta}(X_U, X_R =x_R) = \tilde{y})$ is greater than or equal to $1 -\delta $. 
Without loss of generality, we assume that the representative label is $\tilde{y} =1$. 
Therefore, if we denote $Z$ as the probability of $Pr( f_{\theta}(X_U, X_R =x_R) =1) $, then the probability of $Pr(f_{\theta}(X_U, X_R =x_R) =0) = 1-Z$. 
Additionally, we have $Z \geq 1-\delta > 0.5$ due to the definition of core feature set and by the assumption that $\delta <0.5$. The entropy of the model's prediction can be represented as: 
$H\big[ f_\theta( X_U, X_R = x_R ) \big]  = -Z \log Z - (1-Z) \log (1-Z) $.

Choose $\epsilon(Z) = -Z \log Z - (1-Z) \log (1-Z) $. 
The derivative of $\epsilon(Z)$ is given by $ \frac{d\epsilon(Z)}{dZ} = \log \frac{1-Z}{Z} <0 $, as $Z > 0.5$. As a result, $\epsilon(Z)$ is a monotonically decreasing function, so $\epsilon(Z) \leq \epsilon(1- \delta)$

When $\delta =0$, we have $Z=1$, and by the property of the entropy $ H\big[ f_\theta( X_U, X_R = x_R ) \big]  = 0$. 
\end{proof}

\begin{proposition}
    \label{thm:cond_entropy}
    Given two subsets $R$ and $R'$ of sensitive features $S$, with 
    $R \subseteq R'$, 
    \[
      H\big( f_\theta(X_U, X_R=x_R) \big) \geq 
      H\big( f_\theta(X_{U'}, X_{R'}=x_{R'}) \big),
    \]
    where $U=S\setminus R$ and $U'=S\setminus R'$.
\end{proposition}

\begin{proof}
This is due to the property that conditioning reduces the  uncertainty, or the well-known \emph{information never hurts}  theorem in information theory \cite{krause2005note}.
\end{proof}

\begin{proposition}
\label{app:prop:2} The conditional distribution of any subset of unrevealed features 
$U' \in U$, given the the values of released features $X_R =x_R$ is given by:
\begin{align*}
\Pr(X_{U'} | X_R = x_R)  &= 
    \mathcal{N}\bigg(\mu_{U'}  + \Sigma_{U', R} \Sigma^{-1}_{R R} 
        (x_R - \mu_R),\;
        \Sigma_{U'U'} - \Sigma_{U'R}\Sigma^{-1}_{RR} \Sigma_{R,U'} \bigg),
\end{align*}
where $\Sigma$ is the covariance matrix 
 \end{proposition}

 \begin{proof}
 This is a well-known property of the Gaussian distribution and we refer the reader to Chapter 2.3.2 of the textbook \cite{bishop2006pattern} for further details. 
 \end{proof}


\begin{proposition}
\label{app:prop:3}
The model predictions before thresholding, $\tilde{f}_{\theta}(X_U, X_R = x_R) =\theta_U X_U + \theta_R x_R $ 
is a random variable with a Gaussian distribution $\mathcal{N}\big(m_f, \sigma_f \big)$, 
where
\begin{align}
    m_f &= \theta_R x_R + \theta^\top_U \big(  \mu_U + \Sigma_{U R} \Sigma^{-1}_{R R} ( x_{R} - \mu_R) \big) \label{eq:app_prop4_a}\\ 
    \sigma^2_f &= \theta^\top_{U} \big( \Sigma_{U U} - \Sigma_{U R} \Sigma^{-1}_{R R} \Sigma_{R U} \big) \theta_{U}, \label{eq:app_prop4_b}
\end{align}
where $\theta_U$ is the sub-vector of parameters $\theta$ corresponding
to the unrevealed features $U$. 
 \end{proposition} 

\begin{proof}
The proof of this statement is straightforward due to the property that a linear combination of Gaussian variables $X_U$ is also Gaussian. Additionally, the posterior distribution of $X_U$ is already provided in Proposition  \ref{prop:2}.
\end{proof}

\begin{proposition}
\label{app:prop:4}
Let the model predictions prior thresholding $\tilde{f}_\theta(X_U, X_R =x_R)$, 
be a random variable following a Gaussian distribution $\mathcal{N} (m_f, 
\sigma^2_f)$. Then, the model prediction following thresholding $f_{\theta}( X_U, X_R =x_R)$ 
is a random variable following a Bernoulli distribution $Bern(p)$ 
with $p = \Phi( \frac{m_f}{\sigma_f})$, where $\Phi(\cdot)$ refers to 
the  CDF of the standard normal distribution, and $m_f$ and $\sigma_f$,
are given in Equations \eqref{eq:prop4_a} and \eqref{eq:prop4_b}, respectively.
\end{proposition}

\begin{proof}
In the case of a binary classifier, we have $f_{\theta}(x) = \bm{1}\{ \tilde{f}_{\theta}(x) \geq 0 \}$. 
If $\tilde{f}$ follows a normal distribution, denoted as $\tilde{f} \sim \mathcal{N}(m_f, \sigma^2_f)$, then by the properties of the normal distribution, $ f{\theta} $ follows a Bernoulli distribution, denoted as $f_{\theta} \sim Bern(p)$, with parameter $p =\Phi( \frac{m_f}{\sigma_f})$, where $\Phi(\cdot)$ is the cumulative density function of the standard normal distribution.
\end{proof}

\begin{proposition}
\label{app:thm:core_linear_verf}
Assume $f_\theta$ is a linear classifier. Then, determining if a 
subset $U$ of sensitive features $S$ is a \emph{pure} core feature set 
can be performed in $O(|P| + |S|)$ time.
\end{proposition} 

\begin{proof}
As discussed in the main text, to test if a subset $U$ is a core feature set or not, we need to check  if the following two terms have the same sign (either negative or non-negative):
\begin{align}
\begin{split}
\max_{X_U} \,  & \theta^\top_U X_U + \theta^\top_R x_R = \|\theta_U\|_1 + \theta^\top_R x_R \\
\min_{X_U} \, & \theta^\top_U X_U + \theta^\top_R x_R = - \|\theta_U\|_1 + \theta^\top_R x_R.
\label{eq:min_max_linear}
\end{split}
\end{align}
These can be solved in time $O( |P| + |S|)$ due to the property of the linear equality above.
\end{proof}

\begin{theorem}
\label{thm:taylor_approx}
The distribution of the  random variable $\tilde{f}_\theta = \tilde{f}_{\theta}(X_U, X_R =x_R)$ where $X_U \sim \mathcal{N}\big(\mu^{\mbox{pos}}_{U}, \Sigma^{\mbox{pos}}_{U} \big)$ can be  approximated by a Normal distribution as
\begin{align}
    \tilde{f}_\theta \sim \mathcal{N} \big( \tilde{f}_{\theta}( X_U = \mu^{pos}_U, X_R = x_R) , g^\top_{U}\Sigma^{pos}_{U}g_U\big) 
    \label{eq:taylor_arppox}
\end{align}
where $g_{U} = \nabla_{X_U}\tilde{f}_{\theta}(  X_U = \mu^{pos}_{U}, X_R = x_R)$ is the gradient of model prediction at $X_U = \mu^{pos}_{U}$. 
\end{theorem}

\begin{proof}
The proof relies on the first Taylor approximation of classifier $\tilde{f}$ around its mean: 
\begin{align}
 \tilde{f}_{\theta}(X_U, X_R = x_R, ) &\approx \tilde{f}_{\theta}( X_U = \mu^{pos}_{U}, X_R = x_R)  + (X_U - \mu^{pos}_U)^T \nabla_{X_U}\tilde{f}_{\theta}(X_U = \mu^{pos}_U ,  X_R = x_R)
 \label{app:eq:taylor_approx}
\end{align}
Since $X_U \sim \mathcal{N}\big(\mu^{\mbox{pos}}_{U}, \Sigma^{\mbox{pos}}_{U} \big) $ hence $X_{U} -\mu^{\mbox{pos}}_{U} \sim  \mathcal{N}\big(\boldsymbol{0}, \Sigma^{\mbox{pos}}_{U} \big) $. By the properties of normal distribution, the right-hand side of Equation (\ref{app:eq:taylor_approx}) is a linear combination of Gaussian variables, and it is also Gaussian.
\end{proof}

\section{Algorithms Pseudocode}
\label{app:pseudocodes}
The pseudocode for MinDRel for non-linear classifiers is presented in Algorithm \ref{alg:alg2}. There are two main differences between this algorithm and the case of linear classifiers. Firstly, unlike linear classifiers, the procedure of pure core feature testing on line 5  does not require the guarantee (see again Section \ref{sec:pure_core_non_linear}). The accuracy of the testing procedures depends on the number of random samples that we evaluate. The greater the number of drawn samples, the more likely the testing procedure is to be accurate. During experiments, we draw $10^5$ samples to perform the testing.  Additionally, we use Theorem \ref{thm:taylor_approx} to estimate the distribution of the soft prediction as seen on line 11, as the exact distribution cannot be computed analytically as in the case of linear classifiers.

\begin{algorithm}[t]
  \caption{MinDRel for non-linear classifiers}
  \label{alg:alg2}
  \setcounter{AlgoLine}{0}
  \SetKwInOut{Input}{input}
  \SetKwInOut{Output}{output}
  \SetKwRepeat{Do}{do}{while}
  \Input{A test sample $x$; training data $D$} 
  \Output{A core feature set $R$  and its representative label $\tilde y$}
  $\mu \gets \frac{1}{|D|} \sum_{(x,y) \in D} x$\\
  $\Sigma  \gets \frac{1}{|D|} \sum_{(x,y) \in D} (x -\mu) (x - \mu)^\top$\\ 
  Initialize $R = \emptyset$\\
  \While{True} {
    \eIf{$R$ is a core feature set with repr.~label $\tilde y$}{
      \Return{$(R, \tilde y)$}
    }{
    \ForEach{$j \in U$} {
       Compute $\Pr(X_j | X_R = x_R)$  (using Prop.~\ref{prop:2})\\
      $\bm Z \gets \text{sample}(\Pr(X_j |X_R = x_R))$ T times\\
       Compute $\Pr\left(f_{\theta}(X_j = z, X_{U\setminus \{j\}}X_R = x_R)\right)\!\!\!\!\!\!\!$  \, ( using Theorem \ref{thm:taylor_approx})\\
       Compute $F(X_j)$ (using Eq. \eqref{eq:monte_carlo})
       }
    }
     $j^* \gets \argmax_j F(X_j)$\\
     $R \gets R \cup \{j^*\}$\\
     $U \gets U \setminus \{j^*\}$\\
    }
\end{algorithm}

\section{Extension from binary to multiclass classification}
\label{app:multiclass}
In the main text, we provide the implementation of MinDRel  for binary classification problems.  In this section, we extend the method to the multiclass classification problem.
 
\subsection{Estimating $P(f_{\theta}(X_U, X_R =x_R))$}

In order to achieve our goals of determining if a subset is a core feature set for a given $\delta >0$, and computing the entropy in the scoring function, we need to estimate the distribution of $f_{\theta}(X_U, X_R =x_R)$. In this section, we first discuss the method of computing the distribution of $\tilde{f}_{\theta}(X_U, X_R =x_R)$ for both linear and non-linear models. Once this is done, we then address the challenge of estimating the hard label distribution $P(f{\theta}(X_U, X_R =x_R))$.

It is important to note that, under the assumption that the input features $X$ are normally distributed with mean $\mu$ and covariance matrix $\Sigma$, the linear classifier $\tilde{f}_{\theta} = \theta^\top {x}$ will also have a multivariate normal distribution. Specifically, if $X_U \sim \mathcal{N}(\mu^{pos}_U, \Sigma^{pos}_U)$, then $\tilde{f}_{\theta}(X_U, X_R =x_R) \sim \mathcal{N}(\theta^\top_R x_R + \theta^T_{U} \mu^{pos}_U, \theta^\top_U \Sigma \theta_U)$.

For non-linear classifiers, the output $f_{\theta}(X_U, X_R =x_R) $ is not a Gaussian distribution due to the non-linear transformation. To approximate it, we use Theorem \ref{thm:taylor_approx} which states that the non-linear function $\tilde{f}_{\theta}(X_U, X_R =x_R)$ can be approximated as a multivariate Gaussian distribution.

\paragraph{Challenges when estimating $P(f_{\theta}(X_U, X_R =x_R))$.} 
For multi-class classification problems, the hard label $f_{\theta}(X_U, X_R =x_R)$ is obtained by selecting the class with the highest score, which is given by $\argmax_{i \in [L]} \tilde{f}^i_{\theta}(X_U, X_R =x_R)$. However, due to the non-analytical nature of the $\argmax$ function, even when $\tilde{f}_{\theta}(X_U, X_R =x_R)$ follows a Gaussian distribution, the distribution of $f_{\theta}(X_U, X_R =x_R)$ cannot be computed analytically. To estimate this distribution, we resort to Monte Carlo sampling. Specifically, we draw a number of samples from $P(\tilde{f}_{\theta}(X_U, X_R =x_R))$, and for each class $y\in \cY$ we approximate the probability $P(f_{\theta}(X_U, X_R =x_R) =y)$ as  the proportion of samples that fall in that class $y$.

We provide experiments of MinDRel for multi-class classification cases in Section \ref{sec:app_multi_class_exp}.

\section{Experiments details}
\label{sec:app_exp}

\paragraph{Datasets information.}
To show the advantages of the suggested MinDRel technique for safeguarding feature-level privacy, we employ benchmark datasets in our experiments. These datasets include both binary and multi-class classification datasets. The proposed method was evaluated on the following datasets for binary classification tasks:

\begin{enumerate}
    \item Bank dataset \cite{UCIdatasets}.  The objective of this task is to predict whether a customer will subscribe to a term deposit using data from various features, including but not limited to call duration and age. There are a total of 16 features available for this analysis.
    
    \item Adult income dataset \cite{UCIdatasets}. The goal of this task is to predict whether an individual earns more than \$50,000 annually. After preprocessing the data, there are a total of 40 features available for analysis, including but not limited to occupation, gender, race, and age.
    
    \item Credit card default dataset \cite{UCIdatasets}.  
    The objective of this task is to predict whether a customer will default on a loan. The data used for this analysis includes 22 different features, such as the customer's age, marital status, and payment history.
    
     \item Car insurance  dataset \cite{car_insurance}. 
     The task at hand is to predict whether a customer has filed a claim with their car insurance company. The dataset for this analysis is provided by the insurance company and includes 16 features related to the customer, such as their gender, driving experience, age, and credit score.
\end{enumerate}

Furthermore, we also evaluate the proposed method on two additional multi-class classification datasets:
\begin{enumerate}
\item Customer segmentation dataset  \cite{customer}. 
The task at hand is to classify a customer into one of four distinct categories: A, B, C, and D. The dataset used for this task contains 9 different features, including profession, gender, and working experience, among others.

\item Children fetal health dataset \cite{fetal}. The task at hand is to classify the health of a fetus into one of three categories: normal, suspect, or pathological, using data from CTG (cardiotocography) recordings. The data includes approximately 21 different features, such as heart rate and the number of uterine contractions.
\end{enumerate}

\paragraph{Settings.} 
For each dataset, 70\% of the data will be used for training the classifiers, while the remaining 30\% will be used for testing. The number of sensitive features, denoted as $|S|$, will be chosen randomly from the set of all features. The remaining features will be considered public. 100 repetition experiments will be performed for each choice of $|S|$, under different random seeds, and the results will be averaged. All methods that require Monte Carlo sampling will use 100 random samples. The performance of different methods will be evaluated based on accuracy and data leakage. Two different classifiers will be considered.

\begin{enumerate}
    \item Linear classifiers: We use Logistic Regression as the base classifier.
    \item Nonlinear classifiers: The nonlinear classifiers used in this study consist of a neural network with two hidden layers, using the ReLU activation function. The number of nodes in each hidden layer is set to 10. The network is trained using stochastic gradient descent (SGD) with a batch size of 32 and a learning rate of 0.001 for 300 epochs. 
\end{enumerate}

For Bayesian NN, we employ the package \emph{bayesian-torch} \cite{krishnan2022bayesiantorch} with the default settings. The base regressor is a neural network with one hidden layer that has 10 hidden nodes and a ReLU activation function. We train the network in 300 epochs with a learning rate of 0.001. 

\paragraph{Baseline models.} 
We compare our proposed algorithms with the following baseline models: 
\begin{enumerate}
\item \textbf{All features}: This refers to the usage of the original classifier which asks users to reveal \textbf{all} sensitive features. 
\item \textbf{Optimal}: This method involves evaluating all possible subsets of sensitive features ($2^{|S|}$ in total)  in order to identify the minimum \emph{pure} core feature set. For each subset, the verification algorithm is used to determine whether it is a pure core feature set. The minimum pure core feature set that is found is then selected. It should be noted that as all possible subsets are evaluated, all sensitive feature values must be revealed. Therefore, this approach is not practical in real-world scenarios. However, it does provide a lower bound on data leakage for MinDRel (when $\delta=0$).
\end{enumerate} 

\paragraph{MinDRel models.}
In MinDRel there are two important steps: (1) core feature set verification and (2) selection next feature to reveal. As additional baselines, we keep the core feature set verification and vary the selection process. We consider the following three feature selection methods: 

\begin{enumerate}
\item \textbf{F-Score}: We choose the feature based on the amount of information  on model prediction we gain after revealing one feature as provided in Equation \ref{eq:exp_entropy}. 

\item \textbf{Importance}: We reveal the unknown sensitive features based on the descending order of feature importance until we find a core feature set. The feature importance is determined as follows. We firstly fit a Logistic Regression $f_{\theta}(x) = 1\{ \theta^T x \geq 0 \}$ on the training dataset $D$ using all features (public included). The importance of one sensitive feature $i \in S$ is determined by $\| \theta_i\|_2$.

\item \textbf{Random}: We reveal the unrevealed  sensitive feature in random order until the revealed set is a core feature set.
\end{enumerate}

\paragraph{Metrics.} We compare all different algorithms in terms of accuracy and data leakage:
\begin{enumerate}
\item Accuracy. For algorithms that are based on the core feature set, such as our MinDRel and  Optimal, the representative label is used as the model's prediction. Again, the representative label for $\delta=0$  can be identified by using testing pure core feature set procedures. For $\delta >0$, the representative label is given by $\tilde{y} = \mbox{argmax}_{y \in \cY} \int P(f_{\theta}(X_U =x_U, X_R =x_R) = y) P(X_U |X_R =x_R) dx_u  $.    The accuracy is then determined by comparing this representative label to the ground truth.

\item Data leakage. We compute the percentage of the number of sensitive features that users need to provide on the test set. Small data leakage  is considered better. 
\end{enumerate}

\begin{figure*}[t]
\centering
\begin{subfigure}[b]{0.48\textwidth}
\includegraphics[width = 1\linewidth]{v2_compare_linear_bank.pdf}
\caption{Bank dataset}
\end{subfigure}
\begin{subfigure}[b]{0.48\textwidth}
\includegraphics[width =1\linewidth]{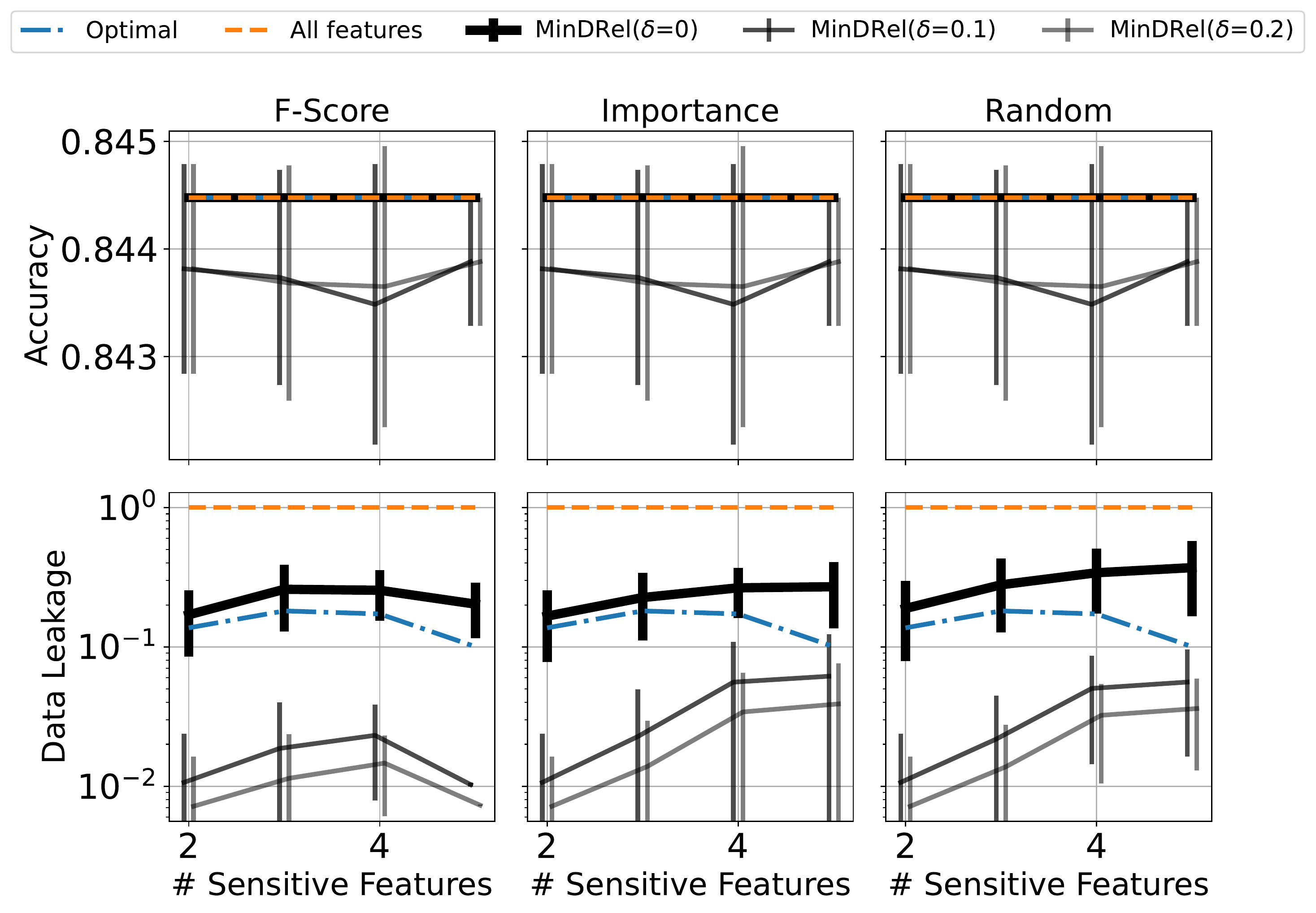}
\caption{Income dataset}
\end{subfigure}
\begin{subfigure}[b]{0.48\textwidth}
\includegraphics[width=1\linewidth]{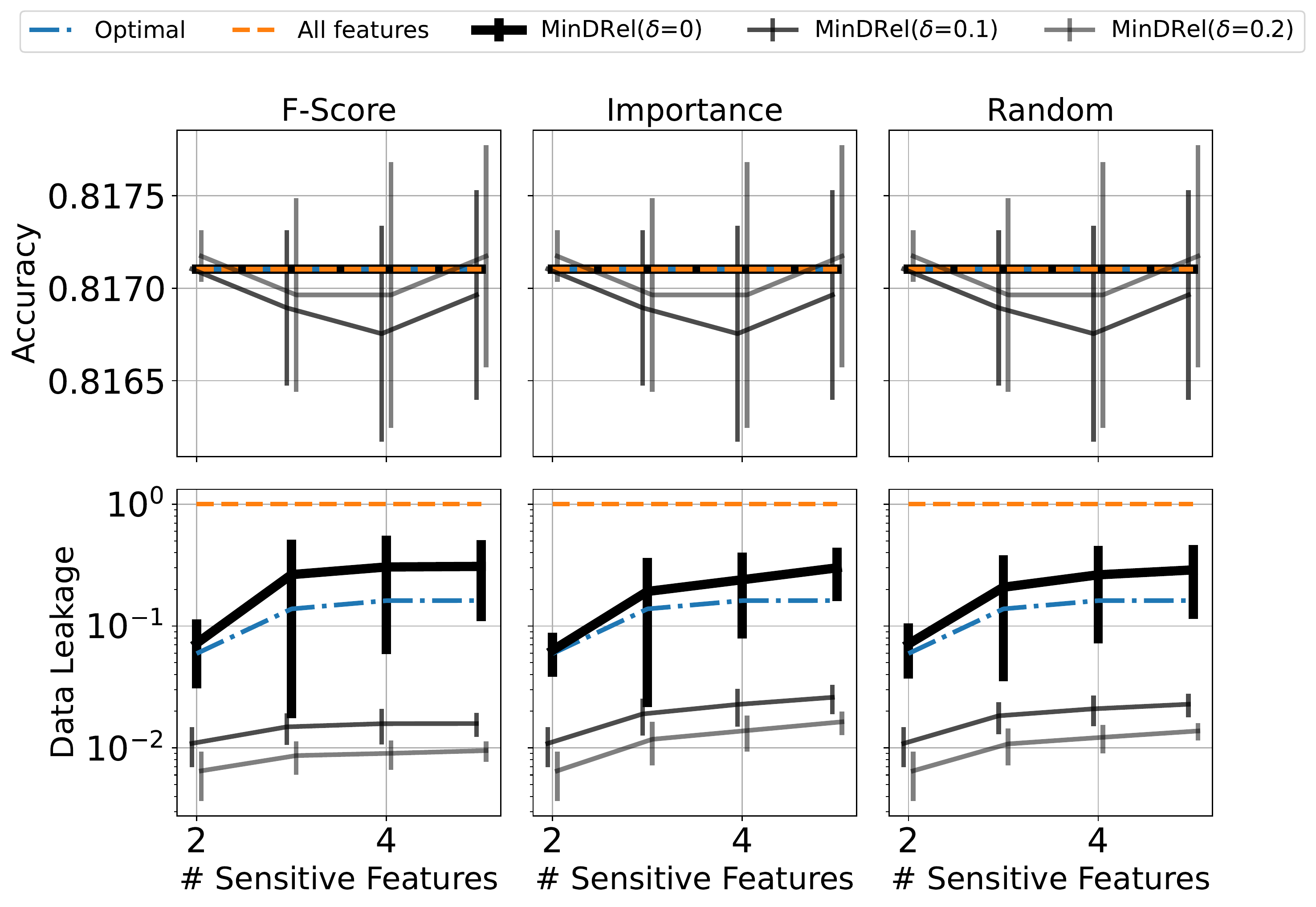}
\caption{Credit dataset}
\end{subfigure}
\begin{subfigure}[b]{0.48\textwidth}
\includegraphics[width = 1\linewidth]{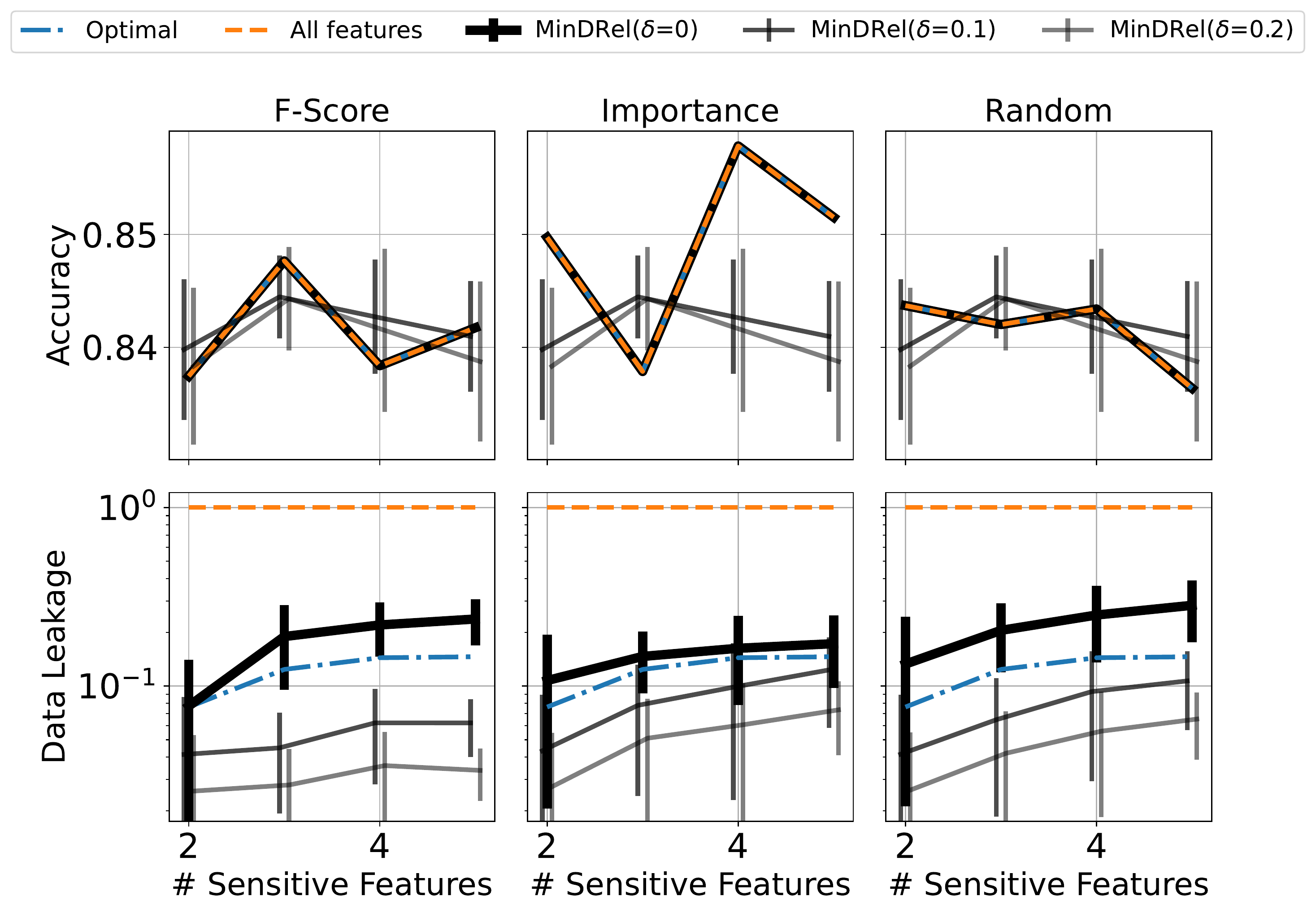}
\caption{Insurance dataset}
\end{subfigure}
\caption{Comparison between using (left) our proposed   F-Score (left) with Importance (Middle) and  Random (Right) for different choices of the number of sensitive features $|S|$.  The baseline classifier is Logistic Regression}
\label{fig:compare_linear_app}
\end{figure*}

\begin{figure*}[t]
\centering
\begin{subfigure}[b]{0.48\textwidth}
\includegraphics[width =1\linewidth]{v2_compare_nonlinear_bank.pdf}
\end{subfigure}
\begin{subfigure}[b]{0.48\textwidth}
\includegraphics[width = 1\linewidth]{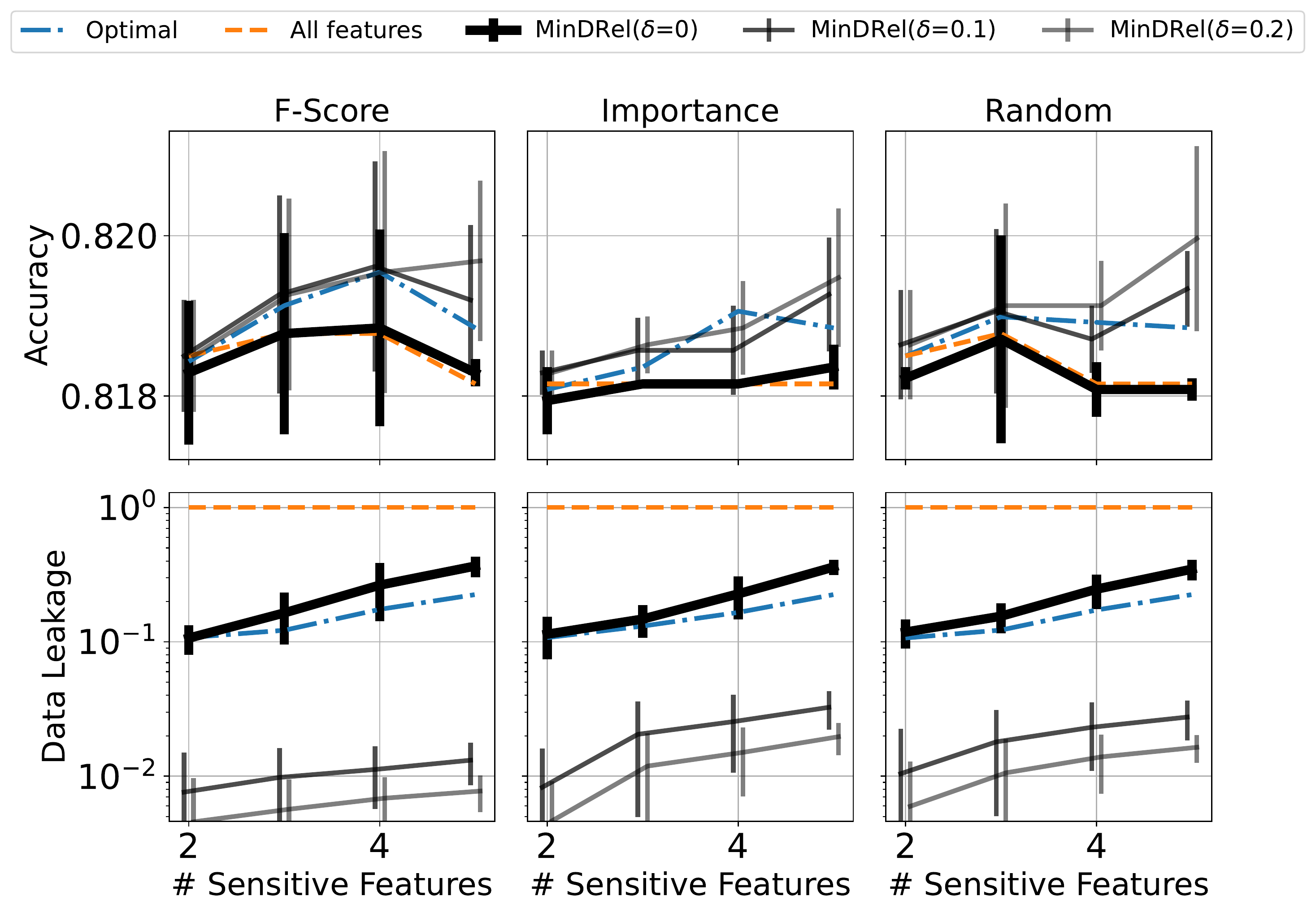}
\end{subfigure}
\begin{subfigure}[b]{0.48\textwidth}
\includegraphics[width =1\linewidth]{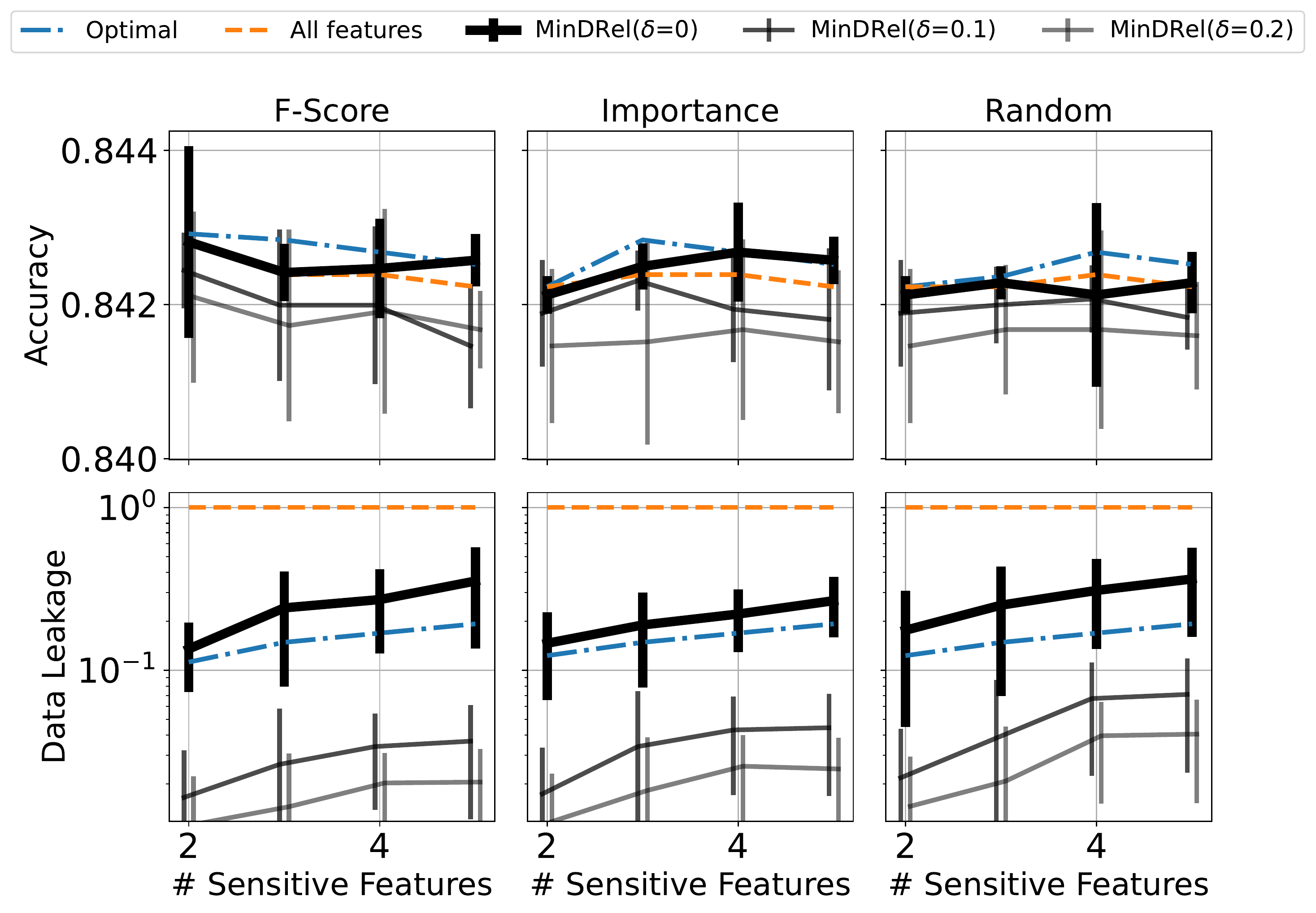}
\caption{Income dataset}
\end{subfigure}
\begin{subfigure}[b]{0.48\textwidth}
\includegraphics[width =1\linewidth]{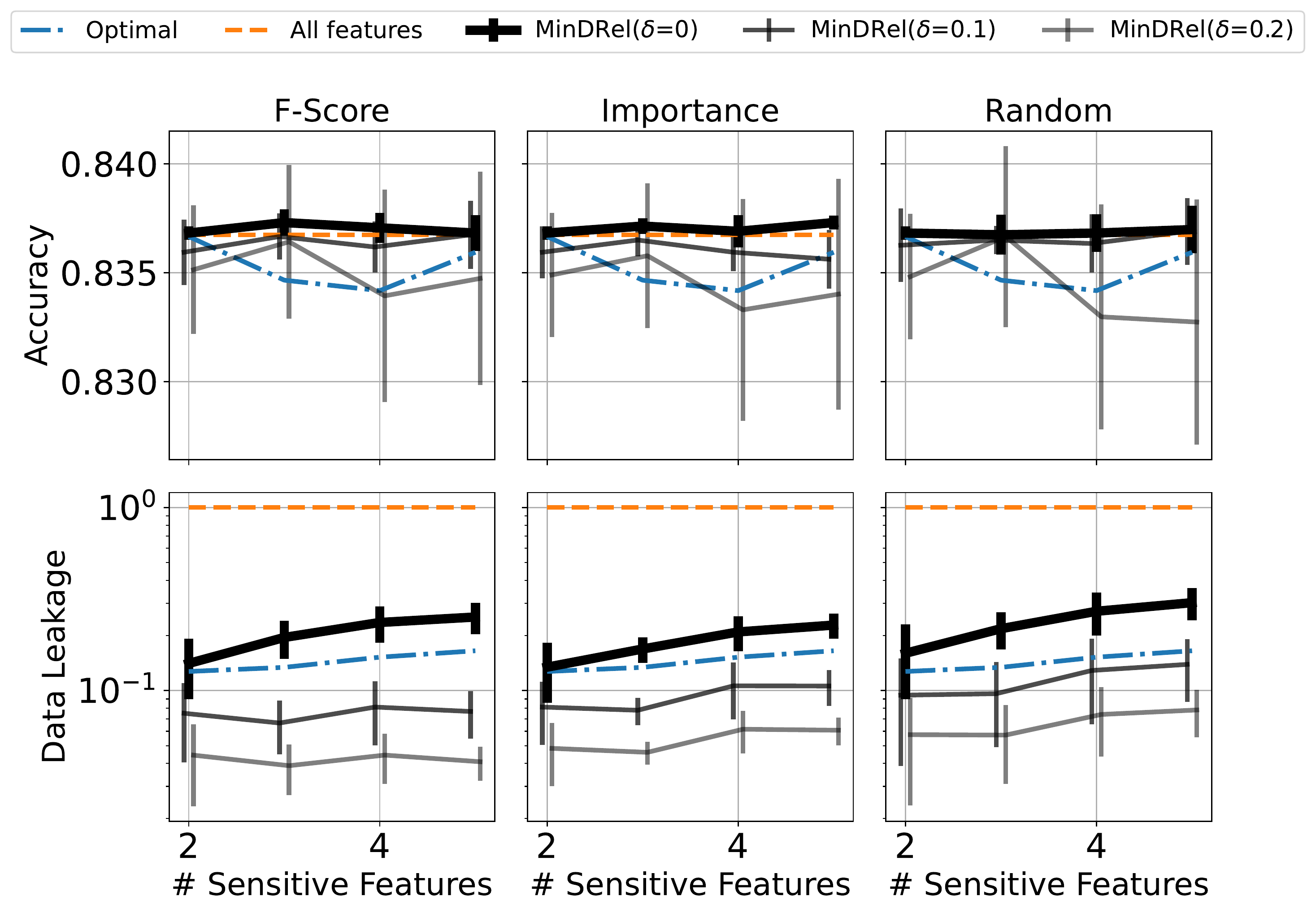}
\caption{Insurance dataset}
\end{subfigure}
\caption{Comparison between using (left) our proposed   F-Score (left) with Importance (Middle) and  Random (Right) for different choices of the number of sensitive features $|S|$. The baseline classifier is  a neural network classifier.}
\label{fig:compare_nonlinear_app}
\end{figure*}

\subsection{Additional comparison between using Gaussian assumption and Bayesian NN}

We first show empirically the benefits of 
 our proposed Gaussian assumption compared to using Bayesian NN which allows more flexibility in  modeling the conditional distribution $P(X_U|X_R =x_R)$. We report both training and inference time between Bayesian NN and our Gaussian assumption on various datasets when the number of sensitive features $|S| = 5$ in Table \ref{tab:1} and Table \ref{tab:2}. When $|S| = 5$ the number of possible subsets $U \in S$ is $2^5 = 32$ which requires training 32 Bayesian NN models. This will be especially slow for datasets with a large number of training samples (e.g., Income with 50K samples). In contrast, using Gaussian assumption we just need to precompute 32 inverse matrices $\Sigma^{-1}_{R,R}$ which is pretty fast for data that have a small number of features (less than 50 in our experiments). It is noted again that in this paper we focus on the case when the number of training samples is much more than the number of features. Likewise, during inference time, with Gaussian assumption, we can compute the distribution of model prediction in a closed form by simple matrix multiplication which takes $O(d^2)$. Instead, using Bayesian NN, it requires  expensive Monte Carlo sampling, especially when $|U|$ is large to obtain an accurate estimation of $P(X_U |X_R =x_R)$. 

 We also report the performance in terms of accuracy and data leakage between using Gaussian assumption and Bayesian NN in Figure \ref{fig:compare_linear_gauss_bayes}. We see no significant difference in terms of accuracy and data leakage between the two choices of modeling $P(X_U |X_R =x_R)$. In addition, as indicated above using the Gaussian assumption reduces significantly the training and inference time, in the subsequent experiments we will use the Gaussian assumption in MinRDel with F-Score selection.

\begin{table}
\centering
\begin{tabular}{c | c c c  c} 
 \toprule
 Method   &  Bank & Income & Credit  & Insurance \\
 \midrule
Bayesian NN & 204 &  375 & 125  & 90 \\ 
Gaussian assumption & 0.01 & 0.02 & 0.02 & 0.01\\
 \bottomrule
\end{tabular}
\caption{Comparison between using Bayesian neural network and our Gaussian assumption in terms of training time (minutes) when |S| = 5 for various datasets.}
\label{tab:1}
\end{table}

\begin{table}
\centering
\begin{tabular}{c|c c c  c}
 \toprule
 Method   &  Bank & Income & Credit  & Insurance \\ 
 \midrule
    Bayesian NN & 40 & 254  & 220 & 34 \\ 
    Gaussian assumption & 15 & 78 & 66 & 9 \\
 \bottomrule
 \end{tabular}
\caption{Comparison between using Bayesian neural network and our Gaussian assumption in terms of inference time (minutes) on the test set when $|S| = 5, \delta =0 $ for various datasets.}
\label{tab:2}
\end{table}

\begin{figure*}[t]
\centering
\begin{subfigure}[b]{0.48\textwidth}
\includegraphics[width = 1\linewidth]
{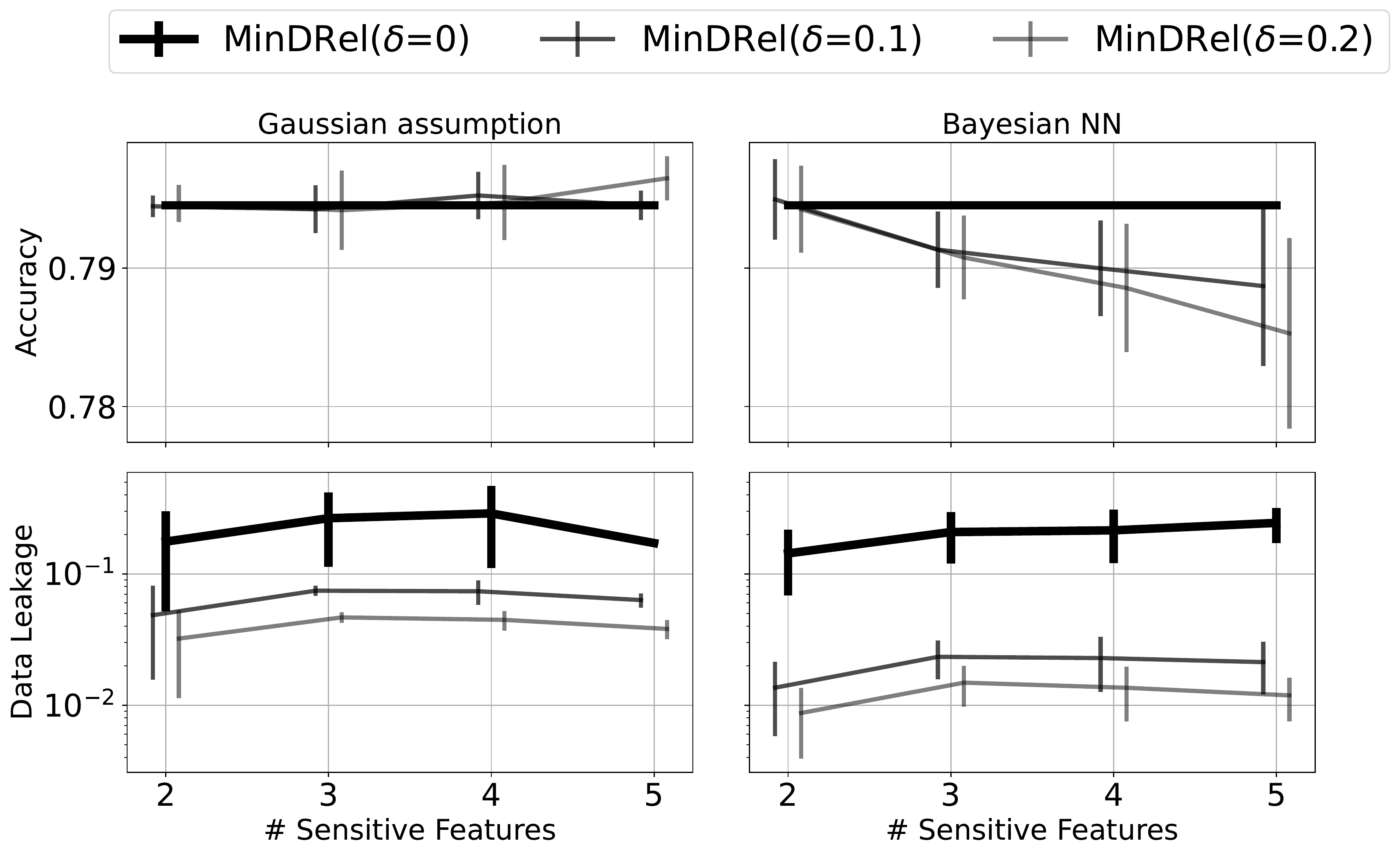}
\caption{Bank dataset}
\end{subfigure}
\begin{subfigure}[b]{0.48\textwidth}
\includegraphics[width = 1\linewidth]{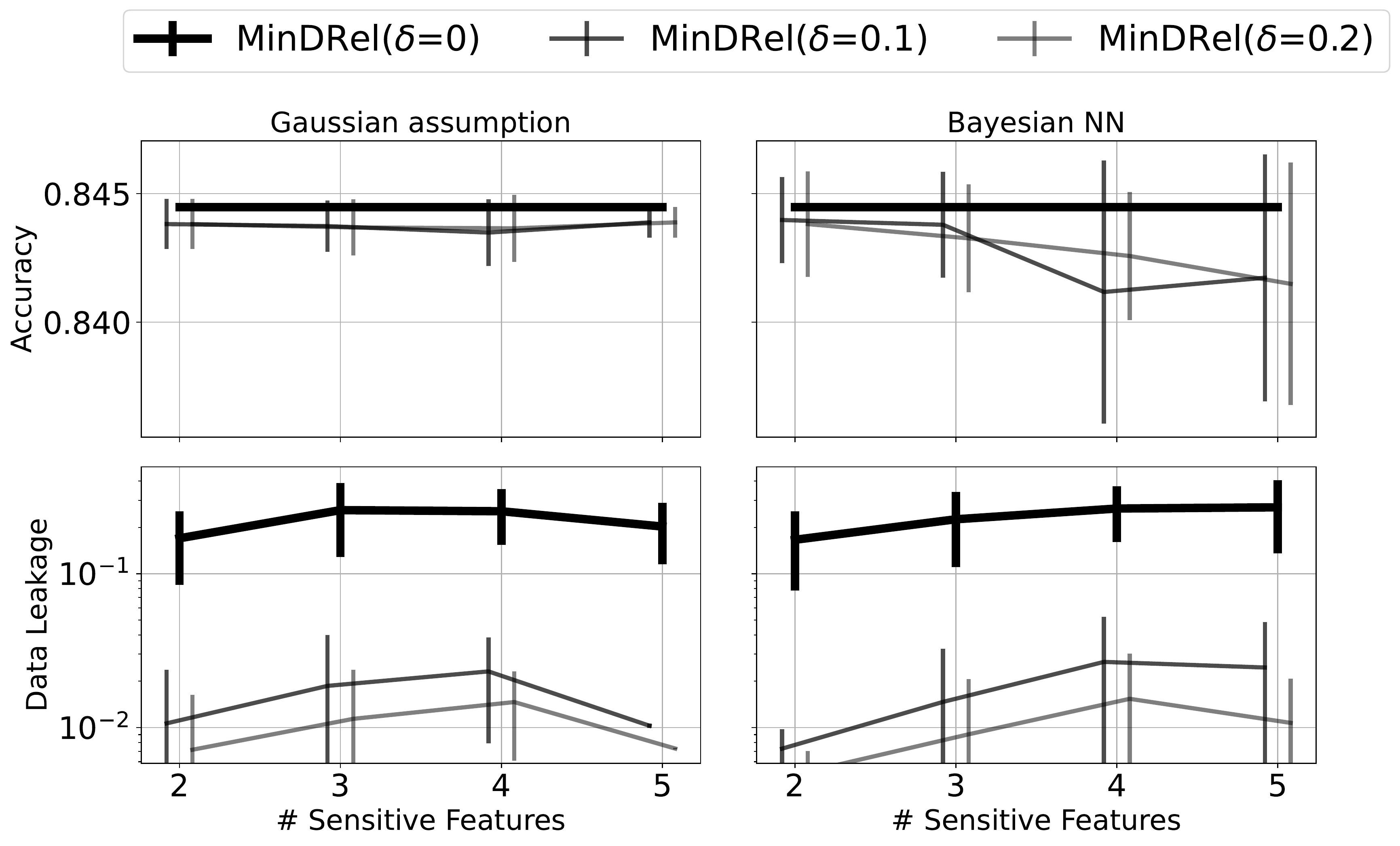}
\caption{Income dataset}
\end{subfigure}
\begin{subfigure}[b]{0.48\textwidth}
\includegraphics[width =1\linewidth]{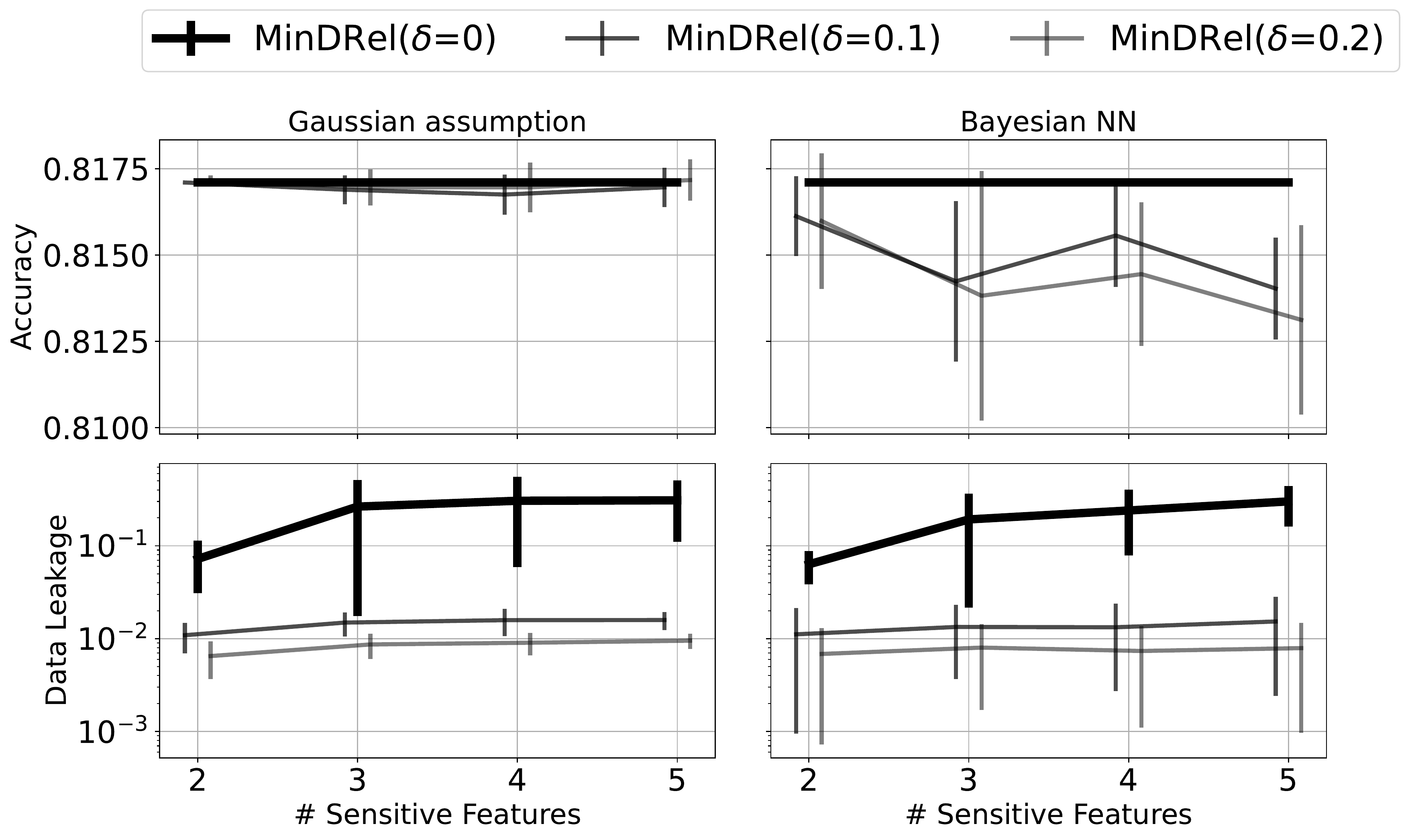}
\caption{Credit dataset}
\end{subfigure}
\begin{subfigure}[b]{0.48\textwidth}
\includegraphics[width =1\linewidth]{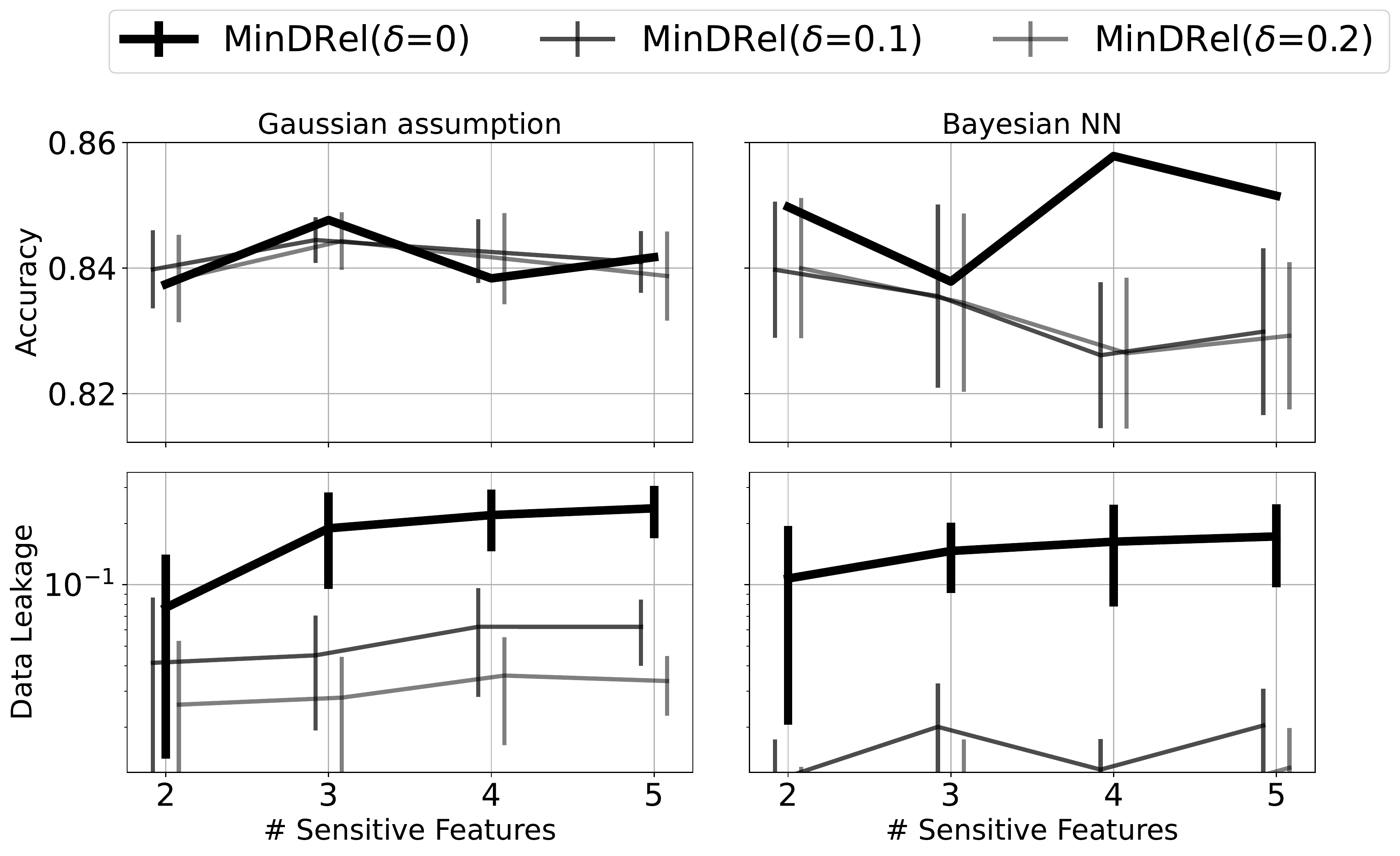}
\caption{Insurance dataset}
\end{subfigure}
\caption{Comparison between using Bayesian NN with our Gaussian assumption in terms of (1): accuracy and (2) data leakage for different choices of the number of sensitive features $|S|$ on different datasets using a Logistic Regression classifier.}
\label{fig:compare_linear_gauss_bayes}
\end{figure*}

\subsection{Additional experiments on linear binary classifiers}

Additional experiments were conducted to compare the performance of MinDRel to that of the baseline methods using linear classifiers on the Bank, Adult income, Credit, and Insurance datasets, as shown in Figure \ref{fig:compare_linear_app}. As in the main text, a consistent trend in terms of performance is observed. As the number of sensitive attributes, $|S|$, increases, the data leakage introduced by MinDRel with various values of $\delta$ increases at a slower rate. With different choices of $|S|$, MinDRel (with $\delta = 0$) requires the revelation of at most 50\% of sensitive information. To significantly reduce the data leakage of MinDRel, the value of $\delta$ can be relaxed. As mentioned in the main text, $\delta$ controls the trade-off between accuracy and data leakage here. The larger $\delta$ is, the greater uncertainty the model prediction has, which implies the fewer sensitive features users need to reveal and the lower accuracy of the model prediction. By choosing an appropriate value for the failure probability, such as $\delta = 0.1$, only minimal accuracy is sacrificed (at most 0.002\%), while the data leakage can be reduced to as low as 5\% of the total number of sensitive attributes.

\subsection{Additional experiments on non-linear binary classifiers}

Additional experiments were conducted to compare the performance of MinDRel to that of the baseline methods using non-linear classifiers on the Bank, Adult income, Credit, and Insurance datasets, as shown in Figure \ref{fig:compare_nonlinear_app}. As seen, while the baseline \textbf{All features} method requires the revelation of all sensitive attributes, MinDRel with different values of $\delta$ only requires the revelation of a much smaller number of sensitive attributes. The accuracy difference between the Baseline method and MinDRel is also minimal (at most 2\%). These results demonstrate the effectiveness of MinDRel in protecting privacy while maintaining a good prediction performance for test data.

\subsection{Sclability of MinDRel for large $|S|$}

We demonstrate the performance of MinDRel when  we have a large number of sensitive features $|S|$. Note that to reduce the runtime we did not run the \emph{Optimal} method which performs   an exponential search over all possible choices of the subset of $S$.

We first report the accuracy and data leakage of MinDRel when using F-Score or using either two heuristic rules Importance and Random in the case of logistic regression classifiers in Figure \ref{fig:compare_linear_large_s}.

\begin{figure*}[t]
\centering
\begin{subfigure}[b]{0.48\textwidth}
\includegraphics[width = 1\linewidth]
{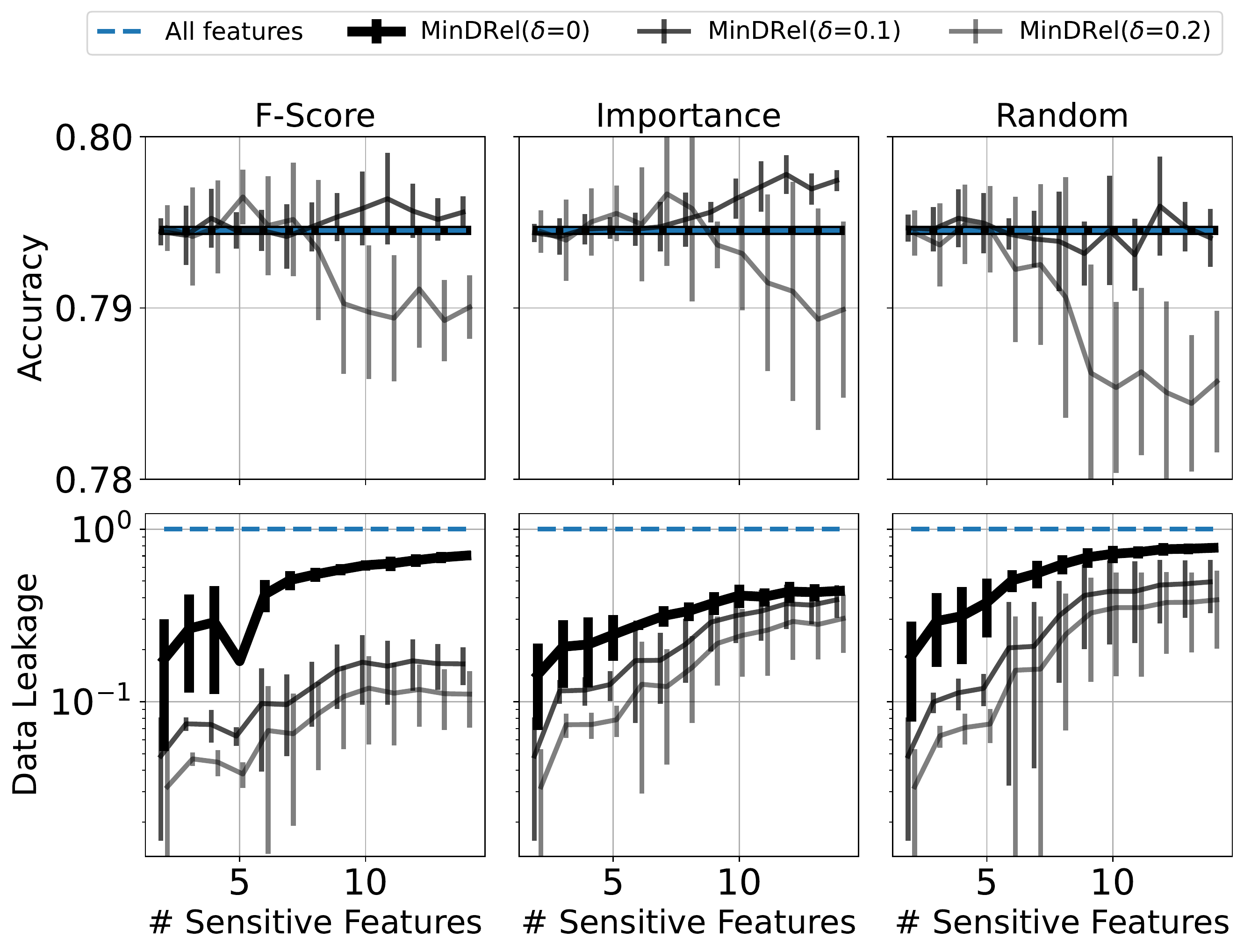}
\caption{Bank dataset}
\end{subfigure}
\begin{subfigure}[b]{0.48\textwidth}
\includegraphics[width = 1\linewidth]{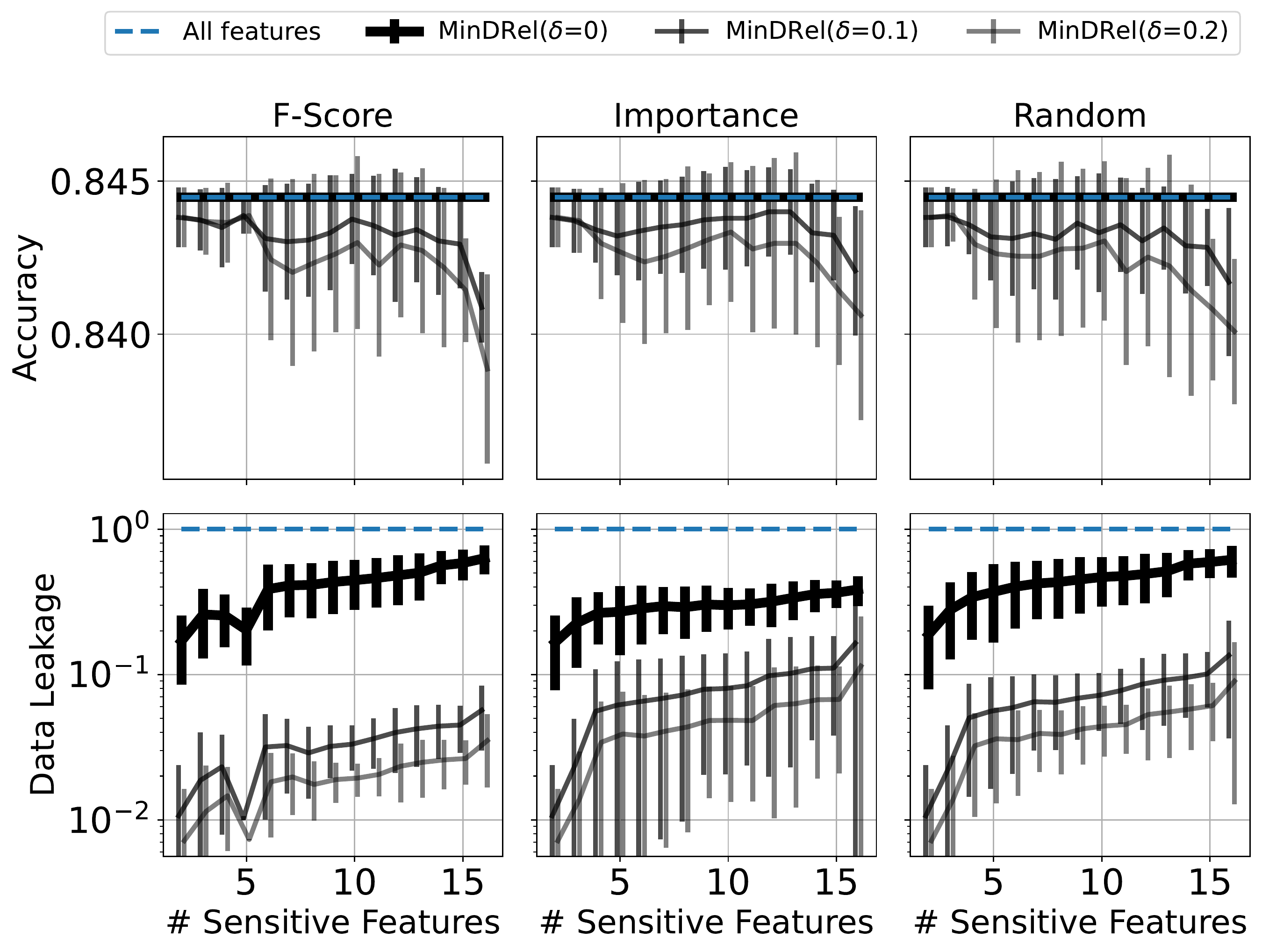}
\caption{Income dataset}
\end{subfigure}
\begin{subfigure}[b]{0.48\textwidth}
\includegraphics[width =1\linewidth]{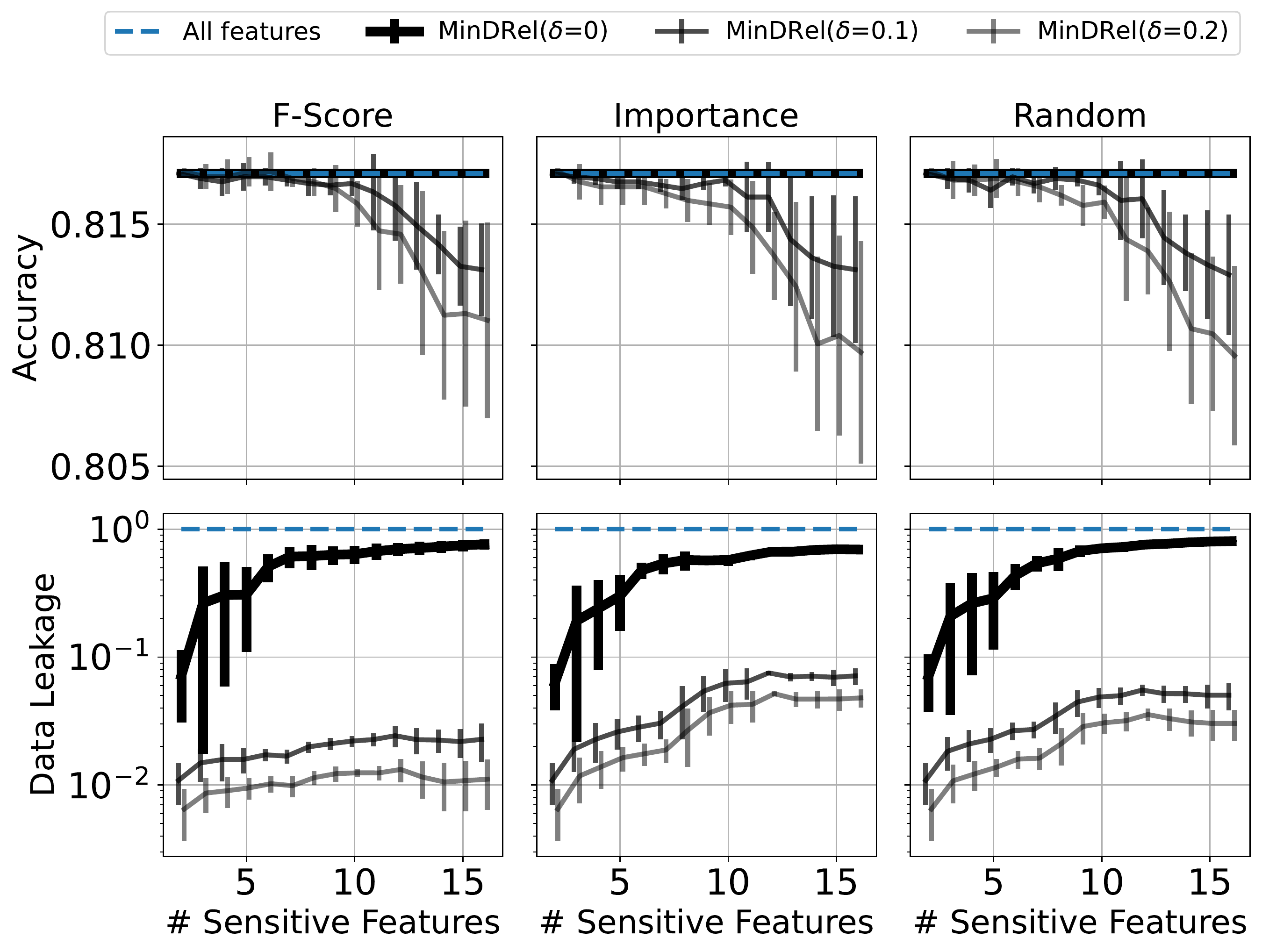}
\caption{Credit dataset}
\end{subfigure}
\begin{subfigure}[b]{0.48\textwidth}
\includegraphics[width =1\linewidth]{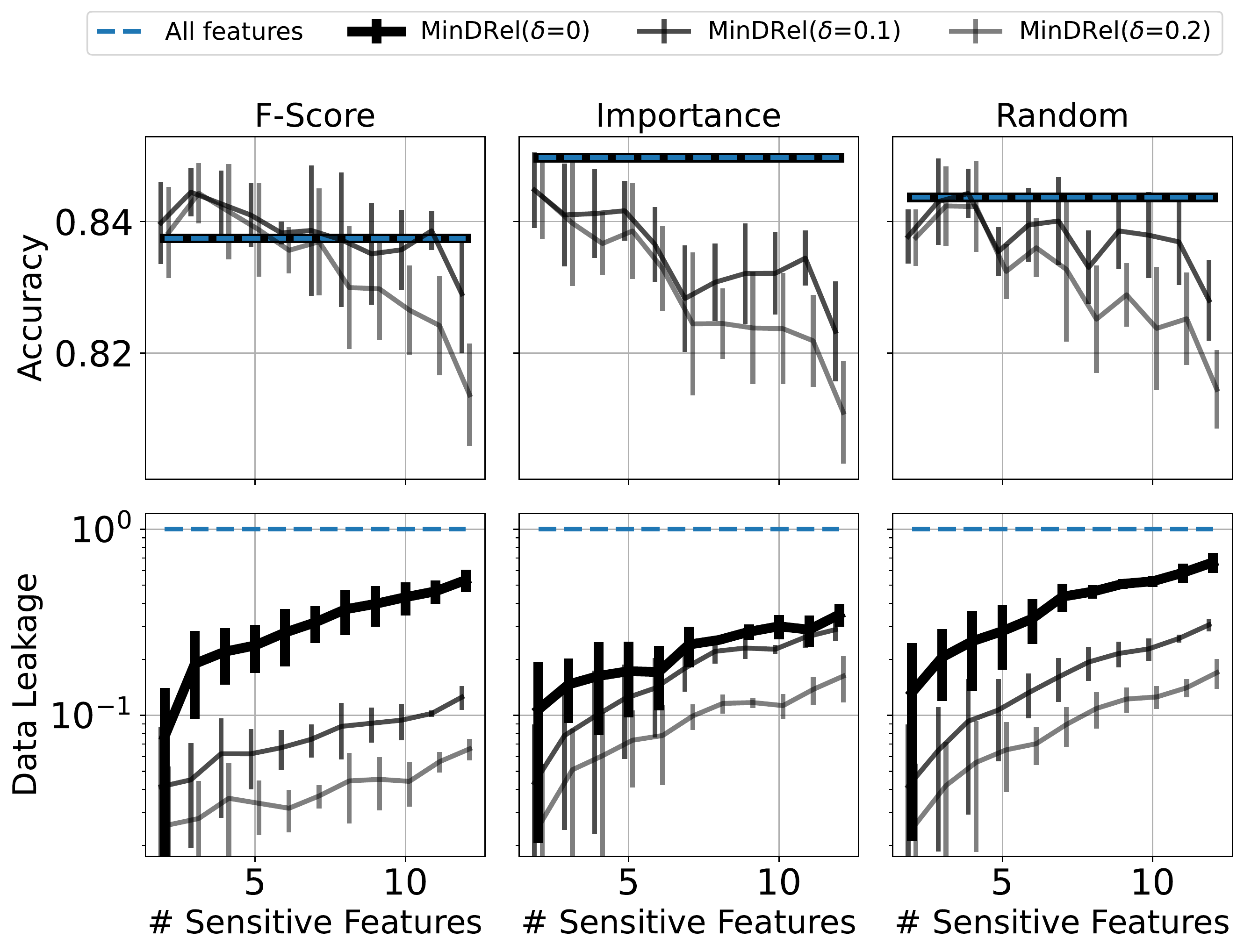}
\caption{Insurance dataset}
\end{subfigure}
\caption{Comparison between using (left) our proposed   F-Score (left) with Importance (Middle) and  Random (Right) for different choices of number of sensitive features $|S|$. The baseline classifier is  a logistic regression classifier.}
\label{fig:compare_linear_large_s}
\end{figure*}

Finally, we report the average  testing time (in seconds) to get the model prediction per user of MinDRel in Figure \ref{fig:time_compare}. It is noted that in this case, we assume the time taken by users to release sensitive features is negligible. It is evident that when $|S| >15$, our  proposed MinDRel with F-Score can take slightly more than 1 second to get the model prediction per user. This demonstrates the applicability of the models in practice.  

\begin{figure*}[t]
\centering
\includegraphics[width =1\linewidth]{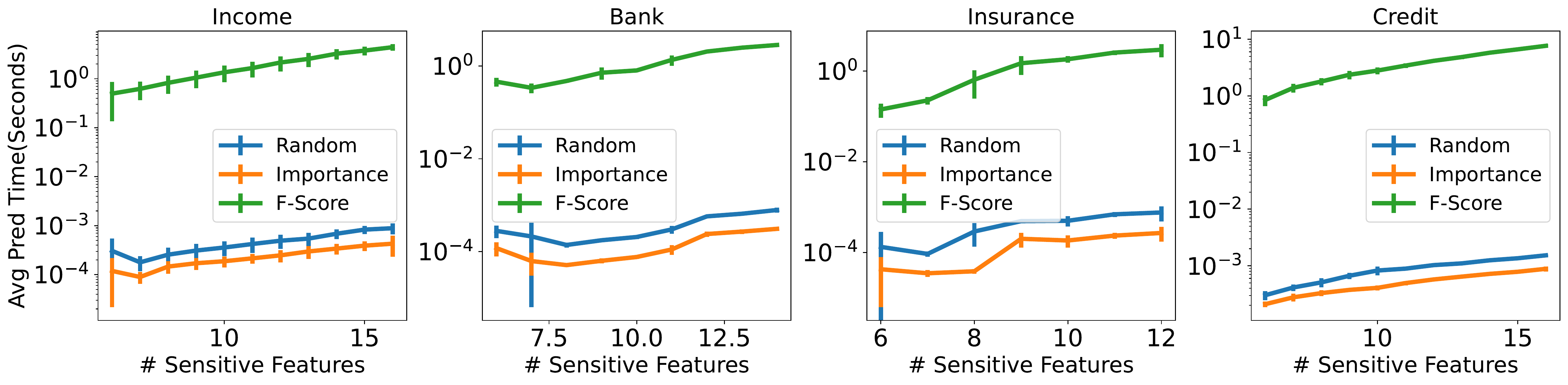}
\caption{Comparison in terms of average prediction time (seconds) among  F-Score, Importance and Random method in MinDRel ($\delta=0$) for different $|S|$. }
\label{fig:time_compare}
\end{figure*}

\subsection{Evaluation of MinDRel on multi-class classifiers}
\label{sec:app_multi_class_exp}
\paragraph{Linear classifiers}
We also provide a comparison of accuracy and data leakage between our proposed MinDRel and the baseline models for linear classifiers. These metrics are reported for the Customer and Children Fetal Health datasets in Figures \ref{fig:multi_class_La} and \ref{fig:multi_class_Lb}, respectively. The figures clearly show the benefits of MinDRel in reducing data leakage while maintaining a comparable accuracy to the baseline models.

\begin{figure*}[tb]
\captionsetup{justification=centering}
\centering
\begin{subfigure}[b]{0.48\textwidth}
\includegraphics[width = 1.0\linewidth]{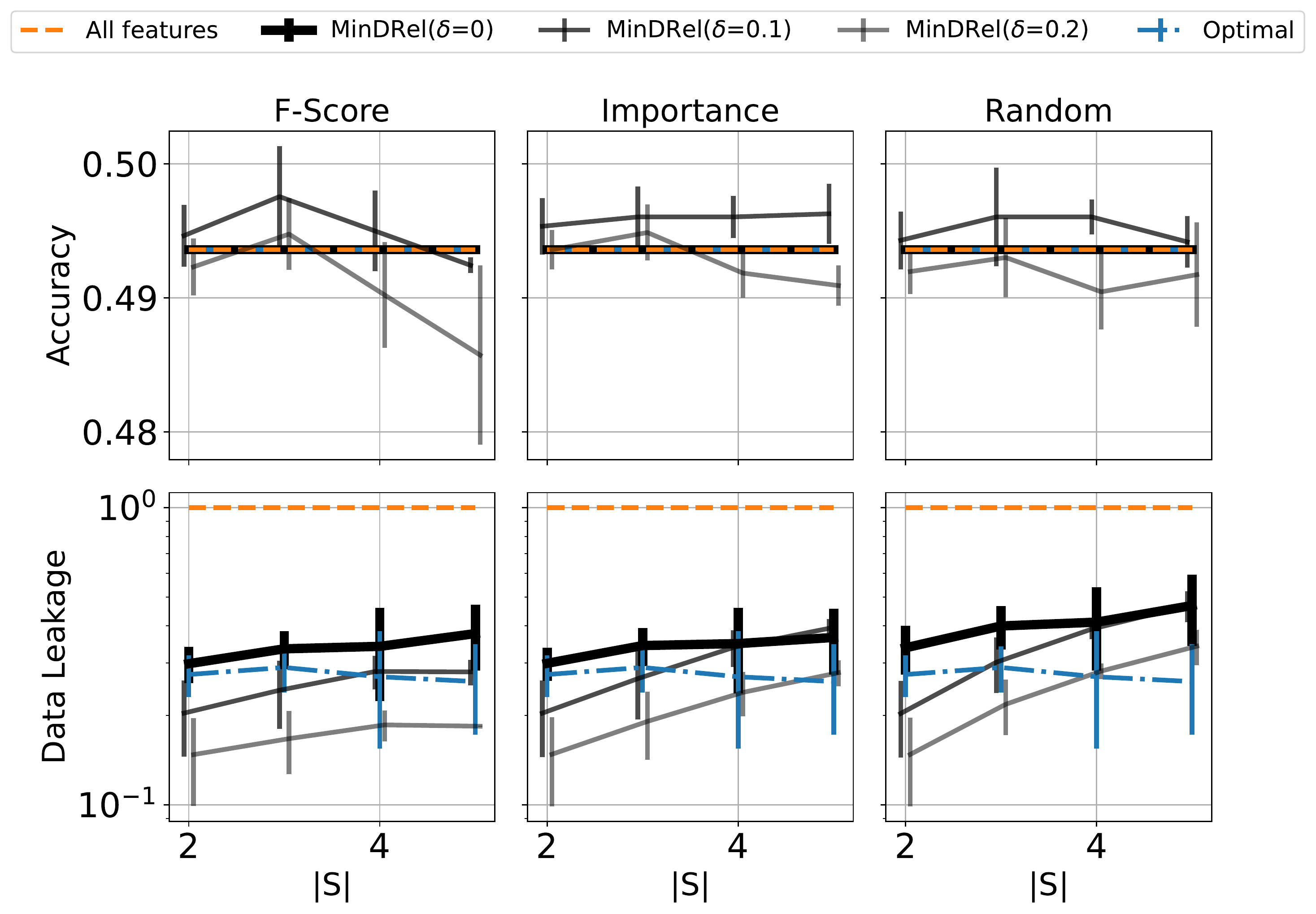}
\caption{Customer dataset}
\label{fig:multi_class_La}
\end{subfigure}
\begin{subfigure}[b]{0.48\textwidth}
\includegraphics[width = 1.0\linewidth]{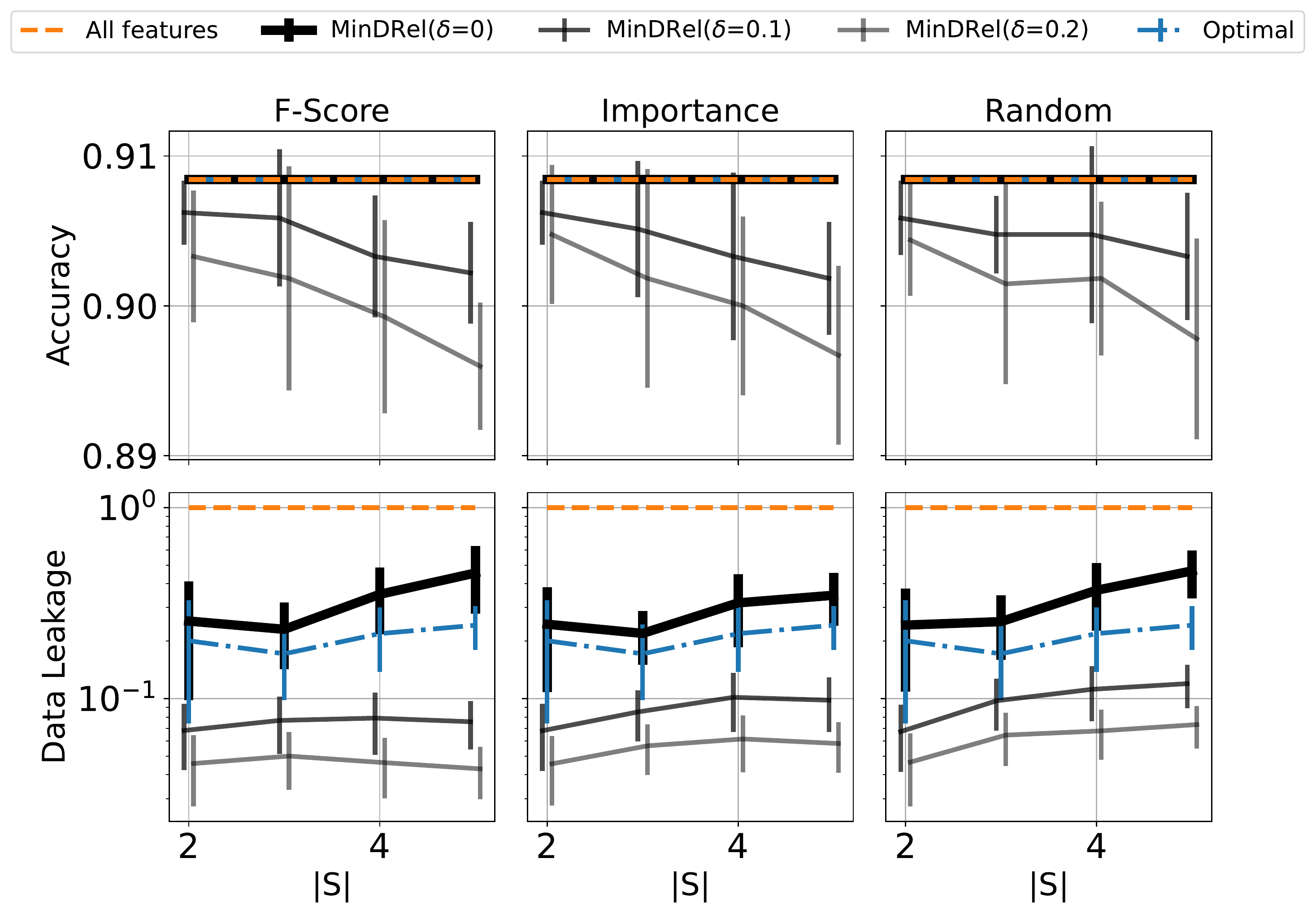}
\caption{Children fetal health dataset}
\label{fig:multi_class_Lb}
\end{subfigure}
\caption{Comparison between using our proposed   F-Score (left) with Importance (Middle) and  Random (Right) for different choices of the number of sensitive features $|S|$. The baseline classifier  is a multinomial Logistic Regression.}
\label{fig:multi_class_linear_v2}
\end{figure*}

\paragraph{Nonlinear classifiers}
Similarly, we present a comparison of our proposed algorithms with the baseline methods when using non-linear classifiers. These metrics are reported for the Customer and Children Fetal Health datasets in Figures \ref{fig:multi_class_a} and \ref{fig:multi_class_b}, respectively. The results show that using MinDRel with a value of $\delta=0$ results in a minimal decrease in accuracy, but significantly reduces the amount of data leakage compared to the Baseline method.

\begin{figure*}[tb]
\captionsetup{justification=centering}
\centering
\begin{subfigure}[b]{0.48\textwidth}
\includegraphics[width = 1.0\linewidth]{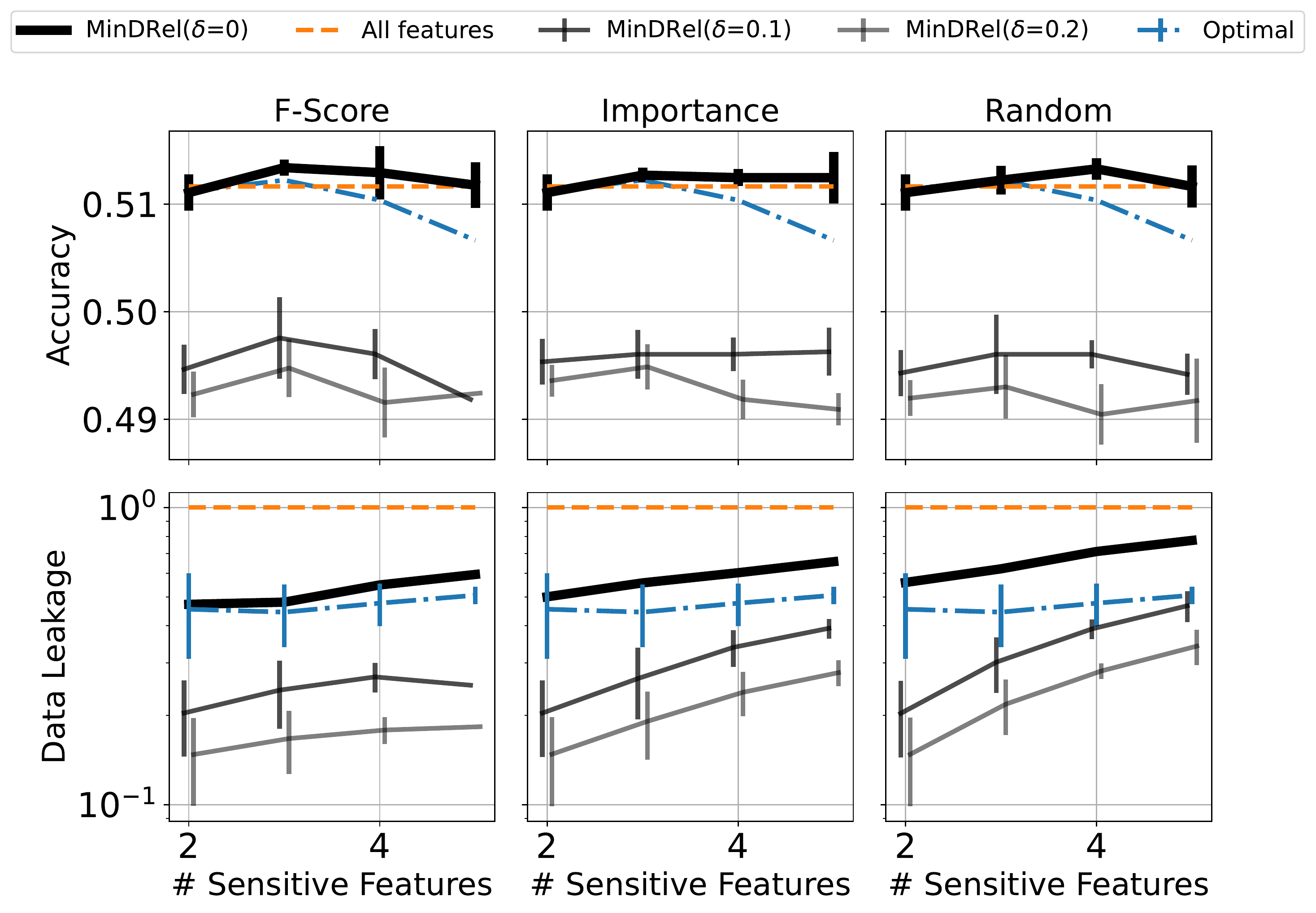}
\caption{Customer dataset}
\label{fig:multi_class_a}
\end{subfigure}
\begin{subfigure}[b]{0.48\textwidth}
\includegraphics[width = 1.0\linewidth]{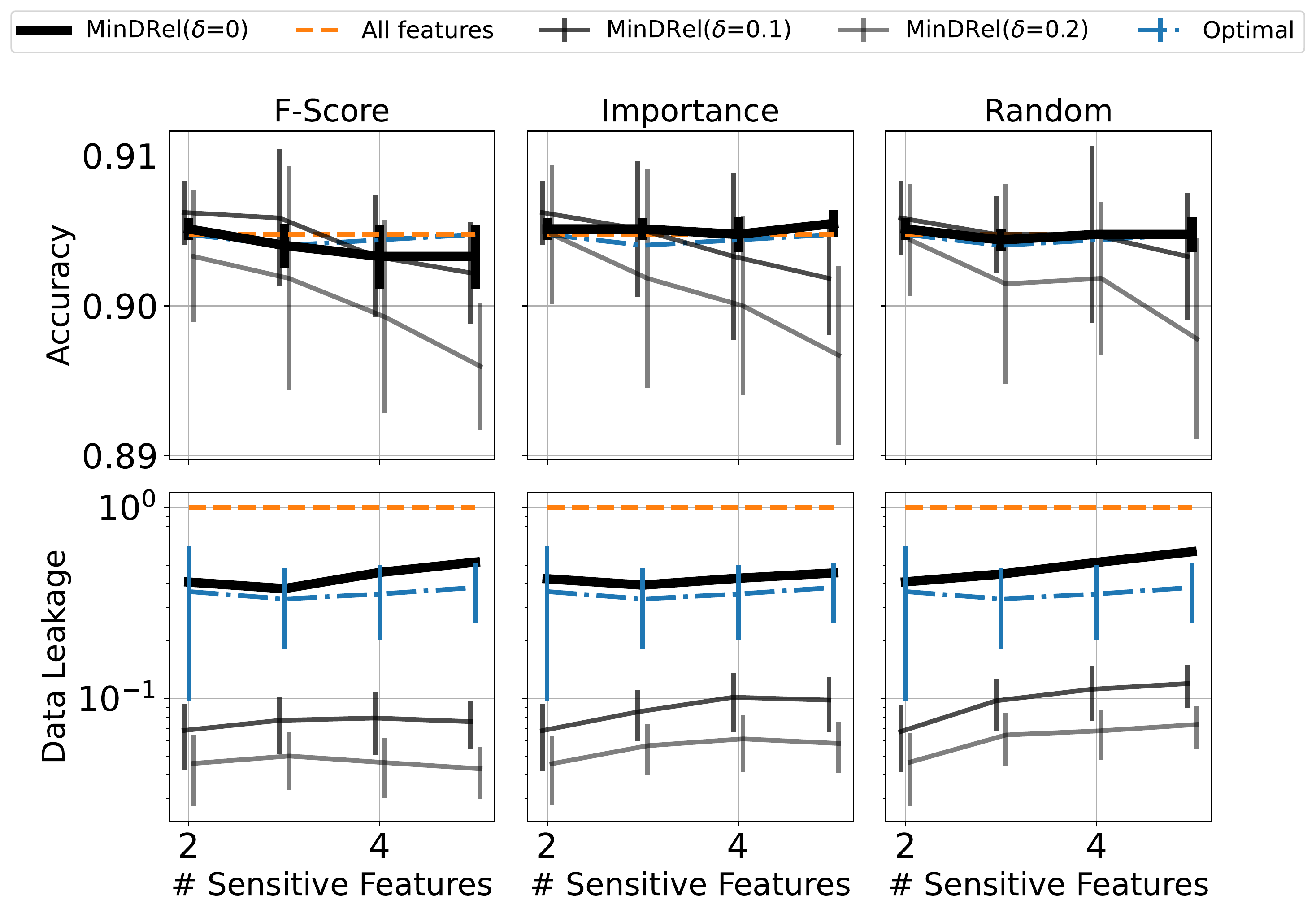}
\caption{Children fetal health dataset}
\label{fig:multi_class_b}
\end{subfigure}
\caption{Comparison between using our proposed   F-Score (left) with Importance (Middle) and  Random (Right) for different choices of the number of sensitive features $|S|$. The baseline classifier is  a neural network classifier.}
\label{fig:multi_class_nonlinear_v2}
\end{figure*}

\end{document}